%% file: arxiv_update_150726.tex
\documentclass[11pt, dvipsnames, DIV=12]{scrartcl}

\usepackage[english]{babel}
\usepackage[utf8x]{inputenc}
\usepackage[T1]{fontenc}
\usepackage{natbib}
\usepackage{tikz}
\usetikzlibrary{arrows.meta,calc,positioning,bending,decorations.pathmorphing,decorations.text}

\usepackage{url}
\usepackage[dvipsnames]{xcolor}
\usepackage{siunitx}
\usepackage{dirtytalk}
\usepackage{caption}
\usepackage{subcaption}
\usepackage{booktabs}
\usepackage{multirow}
\usepackage{paralist}
\usepackage{enumitem}
\usepackage{todonotes}

\usepackage{amsmath}
\usepackage[all,warning]{onlyamsmath}
\usepackage{amssymb}
\usepackage{amsthm}
\usepackage{amsfonts}
\usepackage{thmtools}
\usepackage[euro]{isonums}
\usepackage[mathic=true]{mathtools}
\usepackage{fixmath}
\usepackage{thm-restate}
\usepackage{mleftright}

\usepackage{bm}

\let\originalleft\mleft
\let\originalright\mright
\renewcommand{\mleft}{\mathopen{}\mathclose\bgroup\originalleft}
\renewcommand{\mright}{\aftergroup\egroup\originalright}

\usepackage{algorithm}
\usepackage{algorithmicx}
\usepackage[]{algpseudocodex}

\makeatletter
\def\thmt@refnamewithcomma #1#2#3,#4,#5\@nil{%
\@xa\def\csname\thmt@envname #1utorefname\endcsname{#3}%
\ifcsname #2refname\endcsname
\csname #2refname\expandafter\endcsname\expandafter{\thmt@envname}{#3}{#4}%
\fi
}
\makeatother

\usepackage[pagebackref,
pdfa,
hidelinks,
pdftex,
pdfdisplaydoctitle,
pdfpagelabels,
pdfauthor={Christopher Morris},
pdftitle={},
pdfsubject={},
pdfkeywords={MPNNs, neural algorithmic reasoning, covering ILPs, approximation},
pdfproducer={Latex with the hyperref package},
pdfcreator={pdflatex}
]{hyperref}

\usepackage[capitalise,noabbrev]{cleveref}

\newtheorem{theorem}{Theorem}

\newtheorem{proposition}[theorem]{Proposition}

\newtheorem{lemma}[theorem]{Lemma}

\theoremstyle{definition}
\newtheorem{definition}[theorem]{Definition}
\theoremstyle{remark}
\newtheorem{remark}[theorem]{Remark}
\newtheorem{example}[theorem]{Example}

\usepackage{microtype}
\usepackage{ellipsis}
\usepackage{anyfontsize}
\usepackage{lmodern}

\usepackage[auth-lg]{authblk}

\input{macros}

\title{\LARGE
Graph Learning Should Move Beyond Restrictive Views\\[0.3em]
of Spectral and Message-Passing GNNs}

\author[1]{Antonis Vasileiou}
\author[2]{Juan Cervino}
\author[3]{Pascal Frossard}
\author[4]{Charilaos I. Kanatsoulis}
\author[1]{Christopher Morris}
\author[1]{Michael T. Schaub}
\author[3]{Pierre Vandergheynst}
\author[5]{Zhiyang Wang}
\author[6]{Guy Wolf}
\author[7]{Ron Levie}

\affil[1]{RWTH Aachen University, Aachen, Germany}
\affil[2]{Massachusetts Institute of Technology, Cambridge, MA, USA}
\affil[3]{École Polytechnique Fédérale de Lausanne, Lausanne, Switzerland}
\affil[4]{Stanford University, Stanford, CA, USA}
\affil[5]{University of California San Diego, La Jolla, CA, USA}
\affil[6]{Univ.\ de Montréal; Mila, Montreal, QC, Canada}
\affil[7]{Technion -- Israel Institute of Technology, Haifa, Israel}

\date{}

\recalctypearea

\begin{document}

\maketitle

\begin{abstract}
Graph neural networks (GNNs) are commonly divided into message-passing neural networks (MPNNs) and spectral GNNs, reflecting two largely separate research traditions in machine learning and signal processing. While MPNNs have a precise definition, there is no widely accepted criterion for what makes a mapping a spectral GNN. Most existing work restricts spectral GNNs to layered architectures based on linear spectral filters. Under this restriction, we show that spectral and spatial GNNs have largely equivalent expressive power. To promote progress in the field, we propose a precise definition of spectral GNNs based on eigenbasis symmetries, in contrast to the definition of MPNNs via neighborhood permutation symmetries. We further argue that the two perspectives offer complementary strengths. MPNNs provide a natural language for discrete structure and expressivity analysis through tools from logic and graph isomorphism, while the spectral perspective offers principled tools for understanding smoothing, bottlenecks, stability, and community structure. Overall, we argue that progress in graph learning will be accelerated by clarifying the similarities and differences between these perspectives and by moving toward a unified theoretical framework.
\end{abstract}

\section{Introduction}

Graphs provide a natural representation for relational data arising in domains such as atomistic systems~\citep{Zha+2023c}, drug design~\citep{Won+2023}, optimization~\citep{Cap+2021,Sca+2024,Qia+2024}, and weather forecasting~\citep{Kei+2022}. Over the past decade, graph neural networks (GNNs) have emerged as a central paradigm for learning from such data.

GNN research has largely developed along two separate traditions. One tradition, rooted in machine learning, is based on \new{message-passing neural networks} (MPNNs), in which node representations are updated via permutation-equivariant aggregation over local neighborhoods~\citep{Gil+2017,Sca+2009}. A second tradition, originating in signal processing and applied mathematics, studies \new{spectral GNNs}, which treat node features as graph signals and define neural networks through operators acting in the spectral domain of a \new{graph shift operator} such as the adjacency matrix or graph Laplacian~\citep{Bron+2017,Ortega2018,Isufi2024,gama2020graphs}.

However, despite these fundamentally different perspectives, many architectures proposed as spectral GNNs reduce to variants of message passing. In our view, this has led to a fragmented understanding of the relationship between spectral GNNs and MPNNs. In particular, the lack of a precise and widely accepted definition of spectral GNNs has led to repeated rediscovery of similar results on stability, transferability, and generalization across the two literatures.

\emph{In this position paper}, we argue that spectral GNNs and MPNNs are fundamentally different graph learning frameworks arising from different symmetry principles. MPNNs are naturally defined through local permutation symmetries, whereas spectral GNNs are defined through spectral symmetries of graph operators. Despite this distinction, the lack of a precise and widely accepted definition of spectral GNNs has led much of the literature to adopt restrictive formulations, resulting in the rediscovery of the results discussed above. We argue that this obscures genuinely spectral phenomena and restricts the development of spectral architectures beyond message passing. To address this issue, we advocate for a general definition of spectral GNNs based on eigenprojection symmetries, under which specific choices of filters—such as linear ones—recover architectures equivalent to MPNNs (up to parametrization), while more general choices capture genuinely distinct spectral models.

We further argue that the two perspectives offer complementary advantages. MPNNs provide a natural framework for studying discrete structure, graph isomorphism, and logical expressivity through tools such as the Weisfeiler--Leman hierarchy and logical formalisms~\citep{Cai+1992,barcelo2020,Gro+2021,Mar+2019,Mor+2019,Mor+2022,Xu+2018b}, whereas spectral GNNs provide principled tools for analyzing smoothing, bottlenecks, stability, transferability, and community structure through spectral properties of graph operators~\citep{Gam+2019,levie2021transferability,ruiz2023transferability,cai2020note,black2023understanding,jamadandi2024spectral}.

\paragraph{Present work.}
Our goal is to clarify when spectral GNNs and MPNNs are equivalent, when they differ, and how existing definitions of spectral GNNs have shaped current equivalence results. Concretely, we

\begin{enumerate}
\item propose a general definition of spectral GNNs based on eigenbasis symmetries, contrasting the definition of MPNNs through neighborhood permutation symmetries (see~\cref{subsec:P1});

\item show that, under the commonly used restriction to linear spectral filters, spectral GNNs and standard MPNNs are equivalent in graph distinguishability power and, under bounded spectral assumptions, also equivalent in approximation power, while also explaining how analogous equivalences in transferability results naturally arise (see~\cref{sec:similar_expressivity});

\item claim that MPNNs and spectral GNNs are naturally suited to different theoretical questions, with MPNNs emphasizing discrete structure and logical expressivity and spectral GNNs emphasizing smoothing, bottlenecks, stability, and community structure (see~\cref{subsec:P3});

\item advocate for a clearer separation between spectral filtering and spectral positional encodings, viewing spectral positional information as a general architectural design choice that can benefit both frameworks (see~\cref{subsec:P4}), and which can be constructed either explicitly via eigendecomposition or implicitly through standard GNN message-passing mechanisms.

\end{enumerate}

\paragraph{Related work.}
Several efforts have explored the relationship between MPNNs and spectral GNNs. For example, \citet{Wang+2022} showed that in the transductive setting, spectral GNN expressivity is upper bounded by that of MPNNs, without proving equivalence, while \citet{Guo+2025} argued that spectral GNNs may rely far less on spectral information than commonly believed. \citet{perlmutter2023understanding} proposed a geometric scattering framework relating prominent MPNN and spectral GNN architectures. Prior works have also established connections between spectral and spatial GNNs \citep{DBLP:journals/corr/abs-2003-11702, DBLP:conf/iclr/BalcilarRHGAH21, DBLP:journals/csur/ChenCZJFZCWAL24}. However, these works typically rely on restricted notions of spectral GNNs or on generalized message-passing formulations that require access to spectral quantities. In particular, \citep{DBLP:journals/corr/abs-2003-11702, DBLP:conf/iclr/BalcilarRHGAH21} define spatial GNNs via explicit access to spectral information, while \citet{DBLP:journals/csur/ChenCZJFZCWAL24} study nonstandard MPNNs with direct access to global graph information, making spectral methods a trivial subset of the model class. In addition, existing works mainly establish one-directional inclusion results showing that certain restricted spectral architectures can be implemented through message passing. In contrast, we adopt the standard definition of MPNNs~\citep{Gil+2017} and relate them to the standard definition of spectral GNNs based on linear filters. 
We prove a two-way equivalence in graph distinguishability power (\cref{cor:functional_calculus_expressivity}), extending beyond prior one-directional inclusion results, and characterize when this equivalence does and does not extend to approximation power (\cref{prop:fc_approx_by_poly_bounded_spectrum}).

\begin{figure*}
\begin{center}
    \includegraphics[width=0.99\linewidth]{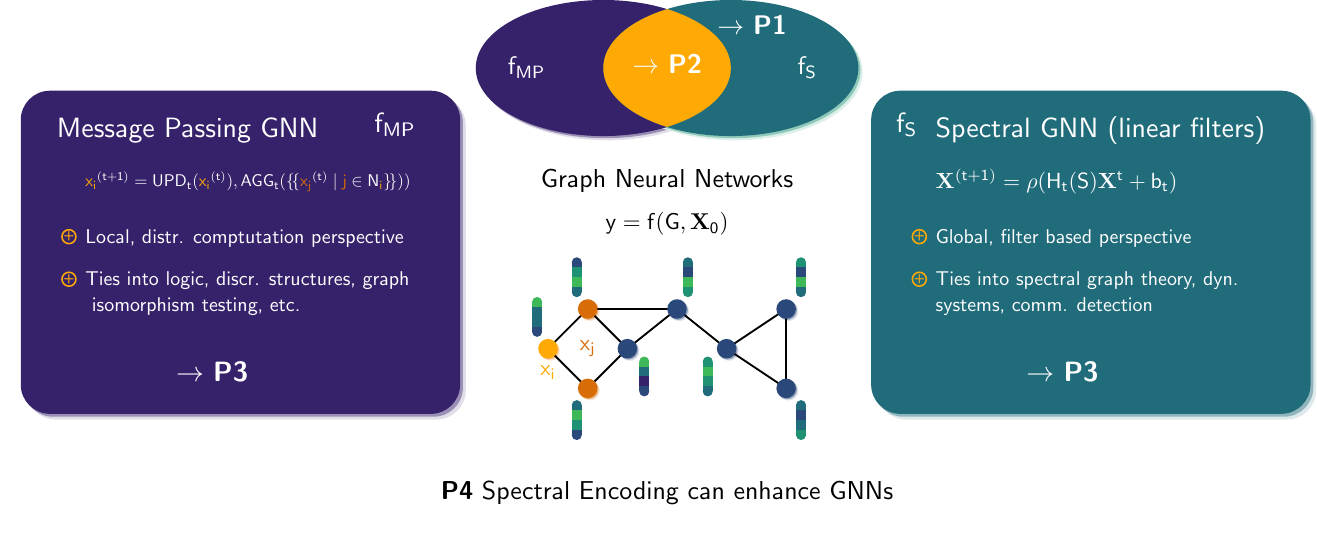}
\end{center}
\caption{Overview of our four positions. \label{fig:overview}}
\end{figure*}

\section{Background}
\label{sec:background}

Here, we introduce the notation and the required basic concepts. Extended background materials are deferred to \cref{app:extensive_background}.

\textbf{Basics}
We denote the set of natural numbers by $\Nb \coloneq \{1,2,\ldots\}$. For $L\in \Nb$, we write $[L]\coloneq \{1,\dots,L\}$. We use $\{\!\!\{\cdot\}\!\!\}$ to denote multisets, i.e., unordered finite collections with possible repetitions. For a matrix $\vec{M}\in\Rb^{n\times m}$, we write $\vec{M}_{i,\cdot}$ (or $\vec{M}[i,\cdot]$) and $\vec{M}_{\cdot,j}$ (or $\vec{M}[\cdot,j]$) for its $i$-th row and $j$-th column, respectively. $\vec{M}_{ij}$ (or $\vec{M}[i,j]$) denotes the $(i,j)$-entry of $\vec{M}$.
We denote by $\text{diag}(\lambda_1,\ldots,\lambda_N)$ the diagonal matrix with diagonal entries $\lambda_1,\ldots,\lambda_N$, and by $\vec{M}^\top$ the transpose of a matrix $\vec{M}$.

\textbf{Graphs}
An undirected graph $G$ is a pair of a finite \new{vertex set} or \new{node set} $V(G)$ and an \new{edge set}
$E(G)\subseteq \{\{u,v\}\subseteq V(G)\mid u\neq v\}$. 
Its \emph{order} is $|V(G)|$. The adjacency matrix $\vec{A}\in\{0,1\}^{N\times N}$ is defined via its entries
$\vec{A}_{uv}=1$ iff $\{u,v\}\in E(G)$ and $\vec{A}_{uv}=0$, otherwise.
The neighborhood of a vertex $u\in V(G)$ is $N_G(u)\coloneqq\{v\in V(G)\mid \{u,v\}\in E(G)\}$, or $N(u)$ when $G$ is clear.
The degree of a node $u$ is denoted $d_u\coloneqq |N_G(u)|$.
Let $\vec{D}\coloneq \text{diag}(d_1,\ldots,d_N)$ be the degree matrix. The \new{combinatorial Laplacian} is defined as $\vec{L}\coloneqq \vec{D}-\vec{A}$, and the \emph{(symmetric) normalized Laplacian} as
$\vec{L}_{\text{sym}}\coloneqq \vec{D}^{-1/2}\vec{L}\vec{D}^{-1/2}$, for an invertible $\vec{D}$.
We consider graphs endowed with node features $\vec{x}_u \in \Rb^{d}$ for $u \in V(G)$ and collect them in the matrix $\vec{X}\in\Rb^{N\times d}$. Two graphs are \new{isomorphic} if there exists a bijection between their vertex sets that preserves adjacency (and node features when present).

A graph neural network (GNN) is any mapping that computes node- or graph-level representations in a vertex permutation-equivariant manner.

\subsection{Message-passing neural networks}
\label{sec:mpnns_main}
\new{Message-passing neural networks} (MPNNs) are GNNs that iteratively update node representations by aggregating information from local neighborhoods. 
Specifically, a generic MPNN updates node features $\vec{x}_u^{(t)} \in \Rb^{d_{t}}$ recursively over a finite number of \new{layer} $L\in \Nb$, as follows
\begin{equation}
\label{eq:mpnn_general}
\!\!\!\!\vec{x}_u^{(t+1)} \!\coloneqq\! \textsf{UPD}_t \mleft( \vec{x}_u^{(t)}, \textsf{AGG}_t \mleft(\{\!\!\{\vec{x}_v^{(t)} \mid v\in N_G(u)\}\!\!\}\mright)
\mright),
\end{equation}
for $u \in V(G)$ and $t\in[L]$. Here $\textsf{AGG}_t$ is a parameterized aggregation map over multisets, $\textsf{UPD}_t$ is a parameterized update map, and $\vec{x}_u^{(0)} \in \Rb^{d_{0}}$ are the initial node features. 
The resulting computations are equivariant to permutations of node orderings and invariant to permutations of neighbors within each neighborhood. For a more extensive motivation and a derivation from symmetry considerations, see \cref{app:mpnns_extensive}. A common and widely used special case of this class is given by sum-aggregation MPNNs (see \cref{eq:deepset_app}).

\subsection{Spectral GNNs Based On Linear Filters}
\label{sec:spectral_gnns_main}

Let $G$ be a graph of order $N$ with node features $\vec{X}$ and let $\vec{S}\in\Rb^{N\times N}$ be a self-adjoint matrix that reflects structural properties of $G$, referred to as a \emph{graph shift operator} (GSO). 
There are several GSO used in the literature; different definitions capture different aspects of the graph structure and lead to distinct modeling properties, as discussed later. 
Common examples include the adjacency matrix $\vec{A}$, the combinatorial Laplacian $\vec{L}$, and the normalized Laplacian $\vec{L}_{\text{sym}}$ (see also \cref{app:extensive_background}). 

Let $\vec{S} = \vec{V}\boldsymbol{\Lambda}\vec{V}^{\top}$ denote the orthonormal eigendecomposition of $\vec{S}$, i.e., $\vec{V} = (\vec{v}_1,\ldots,\vec{v}_N) \in \Rb^{N \times N}$ is an orthogonal matrix whose columns are the eigenvectors of $\vec{S}$, and $\boldsymbol{\Lambda} = \mathrm{diag}(\lambda_1,\ldots,\lambda_N)$ is the diagonal matrix of corresponding eigenvalues. The mapping $\vec{X}\mapsto \vec{V}^{\top}\vec{X}$ is seen as the \emph{graph Fourier transform}, mapping the signal $\vec{X}$ to its frequency (or spectral) representation. The mappings $\vec{Y}\mapsto \vec{V}\vec{Y}$ map the spectral representation $\vec{Y}$ to its spatial representation.

Like MPNNs, spectral GNNs are built as a composition of layers, where each layer consists of a (linear) graph filter followed by a pointwise nonlinearity. A widely used class of such spectral layers is given by \emph{functional-calculus filters}, which apply a learnable function to the spectrum of $\vec{S}$.

Let $\rho$ be a pointwise nonlinearity, and let $H_t \colon \Rb \to \Rb^{d_{t+1} \times d_t}$ be a learnable function at layer $t$. We refer to $H_t$ as a \emph{frequency response}. 
This formulation for frequency response includes architectures based on filter banks (e.g., \citep{Wang+2024a}) and naturally accommodates multidimensional node features.

For a predefined number of layers $L\in\Nb$, a spectral GNN layer takes the form:
\begin{equation}
\label{eq:spec_layer_main}
\vec{X}^{(t+1)}
\coloneqq
\rho\mleft( H_t(\vec{S})\,\vec{X}^{(t)} + \vec{b}_t \mright),
\qquad t \in [L],
\end{equation}
where, $H_t(\vec{S}) \vec{X}^{(t)}
\coloneqq
\sum_{i=1}^{N} \vec{v}_i \vec{v}_i^{\top} \vec{X}^{(t)}\, H_t(\lambda_i)^{\top}$, $\vec{b}_t\in\Rb^{d_{t+1}}$\footnote{Here, addition between $H_t(\vec{S})\,\vec{X}^{(t)}\in\Rb^{N\times d_{t+1}}$ and $\vec{b}_t\in \Rb^{ d_{t+1}}$ is done via broadcasting in the usual way.} is a bias term and $\vec{X}^{(0)}=\vec{X}$.  
Such architecture is called functional calculus spectral GNN, because it defines an operator-valued map $H_t(\vec{S})$ by applying the function to the eigenvalues of $\vec{S}$. Note that the definition of \cref{eq:spec_layer_main} does not depend on the choice of eigenvectors, i.e., the term $\sum_{i:\lambda_i=\lambda}\vec{v}_i\vec{v}_i^{\top}$ is the orthogonal projector onto the eigenspace associated with $\lambda$, which is basis-independent.
Further discussion (including eigenbasis invariance and alternative constructions) is provided in \cref{app:spectral_gnns_extensive}.

A commonly studied subclass of spectral layers is given by \emph{polynomial filters}, where each $H_t$ is a polynomial (see \cref{eq:polynomial_filter}).

\section{Position statements} 
We structure our perspective around the following four key positions.

\subsection{P1: Spectral GNNs and MPNNs Are Derived From Separate Symmetries}
\label{subsec:P1}

Equivariant machine learning is the study of symmetries in data and architectures. Often, one would like to design neural networks that respect data symmetries.  \emph{Equivariant architectures} are models that respect these symmetries by design. This approach typically leads to improved inductive bias and sample complexity, as no data or weights are wasted on learning the symmetries.  While both spectral GNNs and MPNNs are typically described as node permutation invariant or equivariant, this description is not specific enough to uniquely derive either of these classes. We advocate to root MPNNs and spectral GNNs in two different symmetries.

\paragraph{MPNNs From neighborhood permutations} An MPNN can be seen as an architecture that, in each layer, processes each node separately according to its neighborhood. The neighborhood multisets are represented as sequences. Since the same multiset can be ordered in different ways in a sequence, the network that processes the sequences should be \emph{invariant to neighborhood permutations}. Hence, MPNNs are considered architectures that exploit the permutation symmetry of each neighborhood. In fact, one can show that standard message-passing mechanisms in MPNNs are universal approximators for general neighborhood functions. Namely, any continuous function that takes the sequence of node features at a neighborhood, and is invariant to their permutation, can be approximated by one layer of message passing (see Appendix \ref{app:mpnns_extensive} for more details).

\paragraph{Spectral GNNs From eigenbasis invariance} Each layer in a spectral GNN is a computation that maps the signal to the frequency domain, processes its spectral representation, and maps it back to the graph domain. The mapping to the spectral domain is achieved by projecting onto the GSO eigenvectors, and the mapping back is done by taking a linear combination of the eigenvectors, weighted by the spectral coefficients. 
 However, the eigenvectors are not uniquely defined given a GSO. For example, for any eigenvector $\vec{v}_j$, $-\vec{v}_j$ is also an eigenvector. More generally, there is no canonical choice of orthonormal eigenvectors within each eigenspace. Rather, only the eigenspaces are uniquely defined. Hence, a spectral GNN should be invariant to the choice of the eigenbasis. Definition~(\ref{eq:spec_layer_main}) is only a specific example of an eigenbasis invariant GNN, but it is not known to be universal. Instead, we advocate defining general spectral GNNs as follows (written informally).

\begin{definition}
    Each layer in a spectral GNN is a computation that maps the signal to the frequency domain, processes its spectral representation, 
and maps back to the graph domain,  \emph{independently of the specific choice of the eigenbasis (i.e., with eigenbasis invariance)}. 
\end{definition}

We believe that by using this general definition, the community can move beyond the restrictive construction of GNNs based on diagonal linear spectral filters. 

To illustrate that there are eigenbasis-invariant architectures beyond (\ref{eq:spec_layer_main}), we present the following analysis and construction. The standard Fourier inversion formula reads
\begin{equation}
\label{eq:FourierInversion}
   \vec{X} = \sum_{n=1}^N\vec{v}_n\vec{v}_n^{\top}\vec{X}. 
\end{equation}
Here, the $\vec{X}\mapsto \vec{v}_n^{\top}\vec{X}$ is seen as the \emph{analysis} transform, and $\vec{Y}\mapsto \sum_{n=1}^N\vec{v}_n\vec{Y}$ as the \emph{synthesis} transform. A spectral GNN based on this factorization takes the form
\begin{equation*}
\psi(\vec{S},\vec{V},\vec{X})
\coloneqq
\vec{V}\,\hat{\psi} \left(\vec{V}^\top \vec{X}\right),
\end{equation*}
where $\hat{\psi}:\Rb^{N\times d}\to\Rb^{N\times d'}$ is the trainable mapping acting on the spectral coefficients. 
Here, to respect eigenbasis symmetries, standard spectral filters resort to the linear diagonal formula (\ref{eq:spec_layer_main}). Indeed, linearity guarantees that eigenvector sign choices are canceled in analysis--synthesis, and the functional calculus formation assures that spectral components from the same eigenspace are processed identically.

However, one can factor (\ref{eq:FourierInversion}) into analysis and synthesis parts in different ways. For example,  
\begin{equation*}
\vec{X}
=
\sum_{n=1}^N
\Big((\vec{v}_n\vec{v}_n^{\top}\vec{X})./|\vec{v}_n^{\top}\vec{X}|\Big)\,
\Big(|\vec{v}_j^{\top}\vec{X}|\Big)_{j=1}^N.
\end{equation*}
This leads to the spectral GNN layer proposed in \citep{DBLP:conf/nips/LinTL24}
\begin{equation}
\label{eq:NLSF_app}
\psi(\vec{X})
\coloneqq
\sum_{n=1}^N
\Big((\vec{v}_n\vec{v}_n^{\top}\vec{X})./|\vec{v}_n^{\top}\vec{X}|\Big)\,
\psi_n \Big(|\vec{v}_j^{\top}\vec{X}|\Big)_{j=1}^N,
\end{equation}
where $\psi_n:\Rb^{N\times d}\to\Rb^{d\times d'}$ are general trainable mappings. 
This layer automatically respects eigenbasis symmetries, so long as the multiplicity of the eigenvalues is 1.  \citep{DBLP:conf/nips/LinTL24} extended this idea for general multiplicities of the eigenspaces. From another point of view, $(\vec{v}_n\vec{v}_n^{\top}\vec{X})./|\vec{v}_n^{\top}\vec{X}|$ is a consistent choice of the eigenvectors that depends only on the data, and hence (\ref{eq:NLSF_app}) respect the eigenbasis symmetries.  The advantage of such a spectral layer over (\ref{eq:spec_layer_main}) is that it is nonlinear and non-diagonal, i.e., it mixes the information of all frequencies. 

An alternative formulation of an eigenbasis-symmetry-respecting layer is given in \citet{DBLP:conf/nips/GeislerKHG24}. Moreover, eigenbasis symmetries respecting architectures may be achieved via canonicalization \citep{DBLP:conf/icml/KabaMZBR23, DBLP:journals/corr/abs-2509-24886} or potentially be based on frame averaging \citep{DBLP:conf/iclr/PunyASMGBL22}.
See Appendix \ref{app:spectral_gnns_extensive} for more details on the above construction.

\paragraph{Symmetry-preserving universal approximation properties of GNNs}
The MPNN layer formulation in \cref{eq:mpnn_general} is universal in that it can approximate any continuous permutation-invariant function on fixed-dimensional compact domains; see, e.g., \citep{DBLP:conf/nips/ZaheerKRPSS17, Xu+2018b}. In this respect, MPNN layers are sufficiently expressive. In contrast, the spectral GNN layer in \cref{eq:spec_layer_main} does not enjoy a comparable universality guarantee, i.e., it constitutes one particular mechanism for enforcing eigenbasis symmetry, without being known to approximate all continuous functions that respect it. 

    To align with the theoretical guarantees of MPNNs, we argue that the community should seek spectral GNN architectures that simultaneously enforce eigenbasis symmetry and provide universal approximation guarantees. 

 Initial progress in this direction appears in~\citep{DBLP:conf/nips/LinTL24}.

\subsection{P2: MPNNs and Spectral GNNs Based on Linear Filters are Largely Equivalent}
\label{sec:similar_expressivity}
While MPNNs and spectral GNNs originate from different symmetries, spectral GNNS based on linear filters share several theoretical properties with MPNNs.

\paragraph{Polynomial GNNs are special MPNNs}
When spectral frequency responses are restricted to polynomials (\cref{eq:polynomial_filter}) and the GSO is chosen as the adjacency matrix, the resulting spectral layer can be implemented exactly by a linear MPNN with sum aggregation, where each polynomial degree corresponds to one message-passing step. See~\cref{App:P_MPNN} for a formal derivation, including bias terms and extensions to other commonly used GSOs.  
Consequently, analyses of spectral GNNs that focus on polynomial filters often reduce to message-passing arguments and fail to exploit genuinely spectral properties. Such analyses, therefore, do not clarify when spectral methods can outperform MPNNs.

   We advocate studying more general spectral filters, beyond polynomials, using tools from functional analysis, such as \emph{operator Lipschitz continuity}~\citep{aleksandrov2016operator, DBLP:journals/corr/abs-2109-10096}.

\paragraph{The role of the GSO in spectral GNNs with linear filters}
The GSO is central to the expressivity of spectral GNNs. While commonly defined as an $N \times N$ matrix encoding graph structure (e.g., with support on edges and the diagonal), this definition is overly permissive: without further constraints, GSOs can encode arbitrary global graph properties, undermining meaningful analyses of expressivity and complexity.\footnote{For instance, a diagonal GSO encoding isomorphism class enables linear-time graph isomorphism tests.} 
    To avoid such pathologies, we advocate that the community develop and adopt rigorous definitions of GSOs. 

 We believe that promoting rigorous GSO definitions is an important step toward a principled theory of spectral GNNs. The following is an example of a possible definition. 
\begin{restatable}[Aggregation graph shift operator]{definition}{aggregationGSO}
\label{def:aggregationGSO}
Let $G$ be a graph with $N$ nodes.
A matrix $\vec{S} \in \Rb^{N \times N}$ is called an \emph{aggregation graph shift operator} if, for any signal $\vec{f} \in \Rb^N$, the mapping 
$\vec{f}  \longmapsto\; \vec{S}\vec{f}$ can be implemented by a sum-aggregation message passing neural network with finitely many layers applied to the signal $\vec{f}$.
\end{restatable}

Definition \ref{def:aggregationGSO} excludes degenerate constructions while remaining sufficiently general to cover standard choices, e.g., adjacency matrices, Laplacians, and Cayley operators via Jacobi approximations (see \cref{example:aggregation_GSOs}).  This restriction enables principled analysis of expressivity and direct comparison with MPNNs (see Appendices \ref{sec:app_Aggregation_GSO} and \ref{sec:app_expressivity_of_spectral_gnns} for more details).

\paragraph{MPNNs and spectral GNNs have equivalent expressivity in terms of separation power} 
We compare the expressivity of MPNN and spectral GNNs via graph separation, using isomorphism tests that iteratively refine node labels based on local neighborhood structure and induce graph equivalence relations. A GNN class $\mathcal{F}$ is said to match the separation power of a test if any pair of graphs is distinguished by the test if and only if there exists $f \in \mathcal{F}$ producing different outputs on them. Throughout this work, we focus on the \new{$1$-dimensional Weisfeiler--Leman test}~\citep[1-WL;][]{Wei+1968, Wei+1976}, a classical heuristic for the graph isomorphism problem; a formal definition is given in \cref{app:expressivity_as_graph_seperation}. The next result shows that spectral GNNs with continuous filters applied to aggregation GSOs (\cref{def:aggregationGSO}) have separation power exactly matching $1$-WL, which is equivalent to the expressivity of MPNNs of the form \cref{eq:mpnn_general}. This equivalence shows that, once artificially powerful GSOs are excluded, spectral and spatial GNNs have the same expressive power. 

\begin{restatable}[Expressive power of functional-calculus spectral GNNs]{theorem}{expressivePower}
\label{cor:functional_calculus_expressivity}
Let $\mathcal{F}_{\text{C}}(\vec{S})$ denote the class of spectral GNNs defined via functional calculus, i.e., networks whose layers are of the form \cref{eq:SF2}, namely,
\begin{equation*}
\vec{X}^{(t)} = \rho\Big(H_t(\vec{S})\vec{X}^{(t-1)} + \vec{b}^{(t)}\vec{1}^\top\Big),
\end{equation*}
followed by a sum-pooling operator over node features to produce a graph-level representation. Where each frequency response 
$H_t \colon \Rb\to\Rb^{d_t\times d_{t-1}}$ 
is continuous. If $\vec{S}$ is an aggregation graph shift operator, then $\mathcal{F}_{\text{C}}(\vec{S})$ is as expressive as $1$-WL.
\end{restatable}

\paragraph{MPNNs and spectral GNNs have distinct but closely related universal approximation properties}
Beyond graph separation, expressivity can be analyzed via universal approximation on fixed data distributions using the following definition.
\begin{restatable}[Expressivity via universal approximation]{definition}{Expressivityunivapprox}
\label{def:approx}
Let $Q$ be a probability distribution (e.g., over graphs with node features), and let
$F_1$ and $F_2$ be two families of real-valued functions defined on  $Q$.
We say that $F_2$ is \emph{not a universal approximator} of $F_1$ on $Q$ if there exist
$f_1 \in F_1$ and $\varepsilon>0$ such that, for every $f_2 \in F_2$,
\[
\operatorname*{ess\,sup}_{q\sim Q} |f_1(q)-f_2(q)| > \varepsilon.
\]
In that case, we also say that $F_1$ is \emph{not bounded by the approximation capacity of $F_2$ on $Q$}.
\end{restatable}

If $F_1$ is not bounded by the approximation capacity of $F_2$, we say that it is bounded by it.  If both $F_1$ is bounded by the capacity of $F_2$ and vice versa, we say that $F_1$ and $F_2$ have an equivalent approximation capacity.

 If the spectra of the GSOs in the data are uniformly bounded, any spectral GNN can be approximated by a polynomial spectral GNN, because continuous frequency responses on compact intervals admit polynomial approximation. As MPNNs are realizable as polynomial spectral GNNs (\cref{app:approx_mlp_mpnns}), the models have equivalent approximation capacity in this setting (\cref{prop:fc_approx_by_poly_bounded_spectrum}).  
In contrast, for data distributions with unbounded GSO spectra, e.g., complete graphs, whose largest eigenvalue grows with graph size, polynomial spectral GNNs are not universal approximators of spectral GNNs, and consequently neither are MPNNs (see \cref{prop:poly_not_universal_for_fc_adj}). This reveals limitations in expressivity that graph separation alone does not capture and motivates universal approximation as a complementary lens. See \cref{app:expressivity_as_approximation} for more details.

\paragraph{Feature Augmentation Affects Expressivity}
As implied by \cref{eq:mpnn_general,eq:spec_layer_main}, we assume that both MPNNs and spectral GNNs operate directly on graphs with given node features, i.e., each data point is a graph with fixed node features. No feature augmentation is allowed before the network, except for transformations implemented within the layer architecture itself. Under this convention, the expressivity guarantees above apply. Additional feature augmentations, such as structural encodings \citep{DBLP:conf/nips/MaronBSL19} or random features~\citep{Abb+2020, Sat+2020}, can substantially increase expressive power and enable distinguishing otherwise indistinguishable graphs or node configurations; see \cref{subsec:P4}.

\paragraph{Nuances aside, MPNNs and Spectral GNNs with linear filters typically provide similar perspectives} 
Although MPNNs and spectral GNNs are sometimes analyzed through different mathematical tools, many of their theoretical properties are closely aligned in practice. In particular, once differences in architectural conventions are accounted for, the two paradigms often admit nearly equivalent interpretations and guarantees. For stability and size transferability, spectral GNNs are typically analyzed through perturbations of graph filters and graph operators \cite{levie2019transferability,gama2020stability,levie2021transferability,ruiz2021graphon,cervino2022training}, while MPNNs are commonly studied through layer-wise propagation analyses \citep{Chu+2022}. The latter often produce depth-dependent bounds, whereas spectral analyses depend primarily on the degree for polynomial filters, and the Lipschitz constant for continuous filters. Such analyses usually rely on spectral convergence and graph limit arguments\cite{ruiz2021graphon,ruiz2023transferability,wang2023convolutional,wang2024geometric}, whereas MPNN analyses often use operator-theoretic formulations that naturally extend to sparse graph settings \cite{keriven2020convergence,le2023limits,velasco2024graph,Levin2025Transferring,maskey2025generalization}. Despite these technical differences, both frameworks ultimately rely on closely related regularity assumptions and derive qualitatively similar guarantees. However, these distinctions largely disappears when MPNNs defer nonlinearities and apply them only after a block of layers, yielding behavior analogous to polynomial spectral filters. A more detailed discussion appears in \cref{app:p4_extended}.

\label{sec:P4}

\subsection{P3: MPNNs and Spectral GNNs Are Naturally Suited to Different Theoretical Questions}
\label{subsec:P3}

\paragraph{MPNNs have unique benefits for discrete structure and enterprise settings}
\label{sec:P2}
A key advantage of the message-passing perspective is that it aligns GNNs with the rich theory of discrete structures developed in graph theory, logic, and theoretical computer science. For example, MPNNs' expressivity is most naturally characterized via graph isomorphism tests, most notably the $1$-WL test, and logical formalisms such as fragments of first-order logic~\citep{barcelo2020,Grohe23descriptive}, which provide a principled way to reason about what structural properties an MPNN can and cannot capture~\citep{Mor+2019,Xu+2018b}. Such results have led to a clear understanding of MPNNs' limitations and to the principled development of more expressive architectures, overcoming the limitations of MPNNs, e.g., based on the $k$-dimensional Weisfeiler--Leman algorithm~\citep{Cai+1992}, a more expressive generalization of the $1$-WL, subgraph and homomorphism counts~\citep{Bou+2020}, or subgraph-based approaches~\citep{Qia+2022,Fra+2022}. This yields not only stronger empirical models, but also a clear theoretical hierarchy of expressivity, inherited from the corresponding Weisfeiler--Leman and logical hierarchies. That is, in this sense, MPNNs offer both a powerful practical modeling tool and a principled mathematical framework for understanding and extending the limits of graph representation learning. In addition, this viewpoint also directly sheds light on MPNNs' and their more expressive architectures' ability to approximate permutation-invariant and equivariant functions of graphs~\citep{Gee+2020, Azi+2020}.

In addition, MPNNs extend seamlessly to directed graphs~\citep{Ros+2023b} and link prediction tasks via simple, well-understood constructions, e.g., labeling tricks~\citep{Zhu+2021, huang2023a}. In contrast, for example, developing spectral GNNs on directed graphs is much more cumbersome, e.g., leading to the development of the magnetic signed Laplacian~\citep{he2022msgnnspectralgraphneural}. 

Furthermore, the specific local update mechanism of MPNNs has led to close alignment (graph) algorithms~\citep{Cap+2021}. For example, the field of \new{neural algorithmic reasoning}~\citep{Vel+2021} shows that MPNNs are capable of expressing and learning algorithms from data, also leading to results on MPNNs ability to solve linear programs~\citep{chen2023representinglinearprogramsgraph,Qia+2024} and being able to approximate hard maximum constraint satisfaction problems~\citep{yau2024graphneuralnetworksoptimal}.

From a systems and deployment perspective, the explicit aggregation and update structure of MPNNs (see~\cref{eq:mpnn_general}) makes them particularly attractive for large-scale, distributed, and graph learning~\citep[see, e.g.,][]{zhu2019aligraph, zheng2021distdgl, borisyuk2024lignngraphneuralnetworks}. Message passing can be implemented efficiently over partitioned graphs~\citep{Mer+2025} or streamed edges~\citep{Gul+2024}, which is why MPNN-style architectures underlie much of the practical work on GNNs for databases~\citep{relbench} or recommender systems~\citep{Yin+2018a}.

\paragraph{Spectral GNNs have unique benefits for smoothing, bottlenecks, and communities}
To understand how information is smoothed, how it propagates through bottlenecks, and how it aligns with community structure in a graph, spectral GNNs provide a natural abstraction. 
They enable (i) clean characterizations of the oversmoothing phenomenon as spectral contraction, where repeated application of non-expansive filters gradually attenuates high-frequency components and drives representations toward low-dimensional eigenspaces \citep{li2019deepgcns, cai2020note, rusch2023survey}; (ii) intuitive explanations and remedies for oversquashing in terms of spectral quantities such as algebraic connectivity, spectral gaps, curvature and effective resistance, which capture structural bottlenecks and govern long-range information flow \citep{alon2021bottleneck, topping2021understanding, di2023over, black2023understanding, jamadandi2024spectral}; and (iii) community-sensitive representations via low-frequency eigenspaces that encode cluster structure and align with low Dirichlet energy signals \citep{von2007tutorial}.

Beyond interpretability, the spectral perspective makes these phenomena directly analyzable through the frequency response $h(\lambda)$ and the spectrum of the graph shift operator, providing a clearer view of depth–smoothing trade-offs. It also offers practical, theoretically grounded design knobs—such as shaping the passband of spectral filters (e.g., heat kernels, resolvents, or Chebyshev approximations) or modifying graph connectivity to improve spectral properties—that can be used to improve empirical performance \citep{rusch2023survey, epping2024graph, scholkemperresidual}. 

See Appendix~\ref{Ap:Survey of spectral GNNs for smoothing, bottlenecks, and communities} for an extended survey on these topics.

\subsection{P4: Spectral Positional Encodings (SPEs) Are Distinct Tools That Expand Expressivity}
\label{subsec:P4}

Spectral positional encodings (SPEs) play a central role in modern graph Transformers \citep{dwivedi2020generalization, rampavsek2022recipe} and in extending the expressive power of GNNs. They are often implicitly associated with spectral GNNs, since many widely used constructions rely on eigenvectors of GSOs such as the graph Laplacian or random-walk matrix. Concretely, let $\vec{S} =\vec{V}\vec{\Lambda}\vec{V}^\top$ as in \cref{sec:background}. Then the $i-$th row of $\vec{V}$ or $\vec{V}\vec{\Lambda}$ can be used as the PE for the $i$-th node of the graph,
$
\text{SPE}_i = \vec{V}[i,:] ~\text{or SPE}_i = \vec{V}[i,:]\vec{\Lambda}
$.
We argue that this association is conceptually misleading. SPEs constitute an input-representation choice that is orthogonal to whether the backbone architecture is formulated as a message-passing or a spectral operator. Consequently, a substantial fraction of the empirical gains attributed to certain GNN designs is driven by the choice of positional encodings  (PEs) rather than by the parametrization of the graph convolution itself.

Without PEs or other feature augmentations, both MPNNs and spectral GNNs with linear filters are subject to comparable expressivity limitations (\textbf{P1}). In contrast, PEs can inject global structural information or node identifiers that fundamentally change what a model can distinguish. This makes them a dominant factor in performance and generalization, and implies that comparisons between different GNN architectures can be easily confounded unless the PE mechanism is explicitly controlled. From a unifying perspective, PEs should therefore be treated as a separate and explicit design axis, alongside the choice of GNN backbone.

\paragraph{Limitations of SPEs} A central difficulty is that many commonly used SPEs fail to meet basic symmetry and stability requirements. Eigenvectors are defined only up to sign and, in the presence of repeated eigenvalues, up to arbitrary orthonormal transformations within eigenspaces. As a result, standard eigenvector-based PEs are not permutation equivariant and can change drastically under small graph perturbations, potentially harming generalization \citep{davis1969some}. Recent work shows that these limitations can be addressed by enforcing invariance to sign flips and basis rotations, yielding permutation-equivariant and more stable SPEs \citep{limsign,huangstability}.

\paragraph{Generating SPEs with GNNs.} 
  We claim that SPEs need not be introduced solely through explicit eigendecompositions, but can be implicitly generated by standard GNN architectures themselves. In particular, message-passing updates of the form $\vec{X}^{(l)} =  \rho\left(\sum_{k=0}^{K-1}\bm S^k \vec{X}^{(l-1)}\vec{ H}_k^{(l)}\right)$ can be cast as a nonlinear function of eigenvectors $\bm x_v^{(l)} =\rho\left(\bm W^T\bm V[v, :]\right)$ \citep{kanatsoulislearning}, showing that widely used architectures such as GCNs, GINs, and GraphSAGE implicitly construct spectral positional information (for more details see \cref{app:SPE}). Recent works \citep{huangstability, kanatsoulislearning} further demonstrate that spectral graph filters and MPNNs can be used to construct expressive, stable, and permutation-equivariant PEs. Together, these observations suggest a conceptual shift: rather than treating SPEs as a separate preprocessing step, GNNs can be viewed as learning data-dependent spectral representations whose form, expressiveness, and generalization behavior depend jointly on the initial node features \citep{Lou2019,Abb+2020,Sat+2020,kanatsoulis2024graph,kanatsoulis2024counting} and the network parameters. We advocate a modular perspective in which one component of the model constructs equivariant and stable positional information, while another, whether message-passing or spectral, uses this information for downstream prediction. 

\section{Collaboration roadmap and vision}
\emph{Moving forward}, we advocate for a more principled view of graph learning that clearly distinguishes between spectral and spatial (MPNN) GNNs as frameworks arising from different symmetry principles—namely, eigenbasis symmetries for spectral GNNs and neighborhood permutation symmetries for MPNNs. Rather than treating these models as interchangeable, we emphasize understanding when they coincide and when they fundamentally differ. Concretely, we argue for: (i) adopting precise and rigorous definitions of key concepts such as GSOs and spectral GNNs, avoiding implicit assumptions and enabling meaningful comparison across formalisms; (ii) studying spectral GNNs beyond polynomial filters, using tools from functional analysis (e.g., operator Lipschitz continuity) to better capture their behavior; (iii) designing spectral architectures that enforce eigenbasis symmetry while providing desirable properties such as stability and approximation guarantees; (iv) building benchmarks with controlled GSOs that specifically probe oversmoothing, oversquashing, community detection, or directionality/link prediction, enabling us to empirically identify when each viewpoint offers genuine advantages; and (v) clarifying the relationship between spectral filtering and positional encodings, viewing the latter as a general architectural component that can benefit both paradigms.

Overall, our goal is not to unify the two paradigms into a single model class, but to develop a coherent framework that explains their overlap, highlights their differences, and supports principled analysis and design. A detailed discussion of alternative perspectives is provided in \cref{sec:app_alt_views}.

\bibliography{bibliography}

\appendix

\section{Additional related work} 
MPNNs~\citep{Gil+2017,Sca+2009} emerged as one of the most prominent graph machine learning architectures.
Notable instances of this architecture include, e.g.,~\citet{Duv+2015,Ham+2017,Kip+2017} and~\citet{Vel+2018}, which can be subsumed under the message-passing framework introduced in~\citet{Gil+2017}. 
In parallel, many spectral GNNs were introduced in, e.g.,~\citet{Bru+2014,Defferrard2016,Gam+2019,Kip+2017,Lev+2019}, and~\citet{Mon+2017}---all of which descend from early work in~\citet{bas+1997,Gol+1996,Kir+1995,Mer+2005,mic+2005,mic+2009,Sca+2009}, and~\citet{Spe+1997}.

\section{Alternative views}
\label{sec:app_alt_views}
A direct alternative to our position is to posit that MPNN and spectral GNN designs should not be unified. A naïve argument for this view is that, with enough effort, one approach might emerge as superior. However, as discussed in \textbf{P2} and \textbf{P3}, each paradigm has distinct benefits and settings where it provides a more natural starting point.

Another perspective is that competition between the two lines of work fosters productive discussion, even in adversarial form, and thus accelerates progress. However, as discussed in \textbf{P4}, the (somewhat artificial) divide between the two communities often leads to duplication of results, derived with different techniques but ultimately supporting similar conclusions about GNN capabilities. For instance, spectral GNN research frequently focuses on polynomial filters that mimic MPNN behavior (cf. \textbf{P1}), rather than exploring designs that go beyond MPNN constraints.

One might further argue that these approaches originate from fundamentally different backgrounds and rely on incompatible techniques. Yet, the expressivity analysis in \textbf{P1} shows that the function spaces they parameterize have significant overlap, suggesting that a unified perspective is both meaningful and worthwhile.

Beyond feasibility, one may question whether unification should be a primary research focus. While we believe the motivation is compelling—particularly for establishing coherent, parametrization-independent guidelines—we also acknowledge complementary research directions. For example, positional encodings, including spectral ones, have gained attention as a means of enhancing GNN performance independently of architectural choices. As discussed in Sec.~\ref{subsec:P4}, these developments are largely orthogonal to our position and do not diminish the importance of understanding the relationships between GNN paradigms. In fact, trained GNNs can inform data-driven positional encodings~\citep[e.g.,][]{huangstability, kanatsoulislearning, canturk24:PSE}, reinforcing the value of pursuing both directions.

Similarly, one could argue that geometric deep learning should focus on learning transformations that encode graph structure into features, regardless of whether they arise from message passing, spectral filtering, or positional encodings. From this viewpoint, architectures such as graph transformers may eventually subsume hand-designed approaches. However, these models remain challenging to train, often still rely on particular GNN parameterizations, and would benefit from a clearer understanding of information propagation mechanisms on graphs. More broadly, learning geometric representations remains a difficult problem, and a unified framework can provide valuable theoretical guidance.

Finally, one could speculate that advances in large language models may reduce the need for specialized graph learning methods. While such models show impressive capabilities, they do not yet reliably replace tailored graph processing techniques. As such, they do not diminish the relevance of graph machine learning or the importance of developing a principled understanding of GNN architectures.

Overall, we maintain that graph representation learning would benefit from prioritizing a unified formalism that distills the overlap between MPNNs and spectral GNNs—particularly in practical regimes—while preserving their complementary strengths.

\section{Extensive background}
\label{app:extensive_background}

This appendix collects additional notation and mathematical background that is useful for completeness and the analysis in \cref{app:expressivityofmpnns}.

\subsection{Extensive notation}
\label{sec:extensive_notation}

\paragraph{Norms.}
For a vector $\vec{x}\in\Rb^n$, the Euclidean norm and the infinity norm are defined as $\|\vec{x}\|_2 \coloneq \left( \sum_{i=1}^n x_i^2 \right)^{1/2}$ and $\|\vec{x}\|_\infty \coloneq \max_{i\in[n]} |x_i|$. For a matrix $\vec{M}\in\Rb^{n\times m}$, the Frobenius norm is $\|\vec{M}\|_F \coloneq \left( \sum_{i=1}^n \sum_{j=1}^m M_{ij}^2 \right)^{1/2}$. We also use the induced matrix norms $\|\vec{M}\|_2 \coloneq \sup_{\vec{x}\neq \vec{0}} \frac{\|\vec{M}\vec{x}\|_2}{\|\vec{x}\|_2}$ and $\|\vec{M}\|_\infty \coloneq \max_{i\in[n]} \sum_{j=1}^m |M_{ij}|$.

\paragraph{Metric spaces and continuity.}
A \emph{metric space} is a pair $(\cX,d)$ where $\cX$ is a set and $d:\cX\times\cX\to\Rb^+$ satisfies for all $x,y,z\in\cX$: 
(i) $d(x,y)=0$ iff $x=y$, 
(ii) $d(x,y)=d(y,x)$, 
(iii) $d(x,z)\le d(x,y)+d(y,z)$.
Let $(X,d_X)$ and $(Y,d_Y)$ be two metric spaces.
A function $f\colon (X,d_X)\to (Y,d_Y)$ is \emph{$L$-Lipschitz} if
$d_Y(f(x),f(x'))\le L\, d_X(x,x')$ for all $x,x'\in X$. A function $f\colon (X,d_X)\to (Y,d_Y)$ is \emph{uniformly continuous} if for every $\varepsilon>0$ there exists $\delta>0$ such that
$d_Y(f(x),f(x'))<\varepsilon$ whenever $d_X(x,x')<\delta$, for all $x,x'\in X$. Let $(X,d_X)$ and $(Y,d_Y)$ be metric spaces, and let $\mathcal F$ be a family of functions $f\colon (X,d_X)\to (Y,d_Y)$. The family $\mathcal F$ is said to be \emph{equicontinuous} if for every $\varepsilon>0$
there exists $\delta>0$ such that
\[
d_Y\bigl(f(x),f(x')\bigr)<\varepsilon
\quad\text{whenever}\quad
d_X(x,x')<\delta,
\]
for all $x,x'\in X$ and all $f\in\mathcal F$.

\begin{remark}
\label{remark:unif_cont_char}
Uniform continuity can equivalently be characterized by the existence of a non-decreasing function $\omega\colon \Rb^+\to\Rb^+$ with $\omega(0)=0$ and $\lim_{t\to 0}\omega(t)=0$ such that $d_Y(f(x),f(x'))\le \omega \big(d_X(x,x')\big)$ for all $x,x'\in X$. In particular, $L$-Lipschitz functions are uniformly continuous with $\omega(t)=Lt$.
\end{remark}

\begin{remark}
\label{remark:equicontinuity_char}
Equivalently, $\mathcal F$ is equicontinuous if there exists a non-decreasing function
$\omega\colon\Rb^+\to\Rb^+$ with $\omega(0)=0$ and $\lim_{t\to 0}\omega(t)=0$ such that
\[
d_Y\bigl(f(x),f(x')\bigr)\le \omega\bigl(d_X(x,x')\bigr)
\quad\text{for all } x,x'\in X \text{ and all } f\in\mathcal F.
\]
\end{remark}

\paragraph{MLPs}
A multilayer perceptron (MLP) is a feedforward neural network obtained by composing affine maps with pointwise nonlinearities. Given an input $\vec{x}\in\Rb^{d_0}$, an $L$-layer MLP computes
$\vec{x}^{(t+1)}=\rho\left(\vec{W}_t \vec{x}^{(t)}+\vec{b}_t\right)$ for $t\in[L]$,
where $\vec{W}_t$ and $\vec{b}_t$ are weight matrices and bias vectors, and $\rho$ is a non linear function applied componentwise (e.g., ReLu, LeakyReLu, softmax).

\paragraph{Graph shift operators}
Given a graph $G$, a \emph{graph shift operator (GSO)} is any matrix $\vec{S}\in\Rb^{N\times N}$ that in some sense respects the connectivity of the graph. Note that this is not a rigorous definition. 
Some examples of self-adjoint GSOs are the \emph{combinatorial Laplacian} $\vec{L}_{\mathrm{ comb}}=\vec{D}-\vec{A}$, the \emph{symmetric normalzied Laplacian} $\vec{L}_{\mathrm{ sym}}=\vec{D}^{-1/2}\vec{L}_{\mathrm{comb}}\vec{D}^{-1/2}$, and the adjacency matrix itself $\vec{A}$. One can also define non-self-adjoint GSOs, like the \emph{random walk Laplacian} $\vec{L}_{\mathrm{rw}}=\vec{D}^{-1}\vec{L}_{\mathrm{comb}}$ or the unitary \emph{Cayley shift operator} $\vec{C}=(\vec{L}-ic \vec{I})(\vec{L}+ic \vec{I})^{-1}$, where $\vec{L}$ is any self-adjoint GSO and $c>0$.

\paragraph{Graph Fourier transform.}
Any self-adjoint GSO $\vec{S}$ has an orthonormal eigendecomposition
$\vec{S} = \vec{V}\boldsymbol{\Lambda}\vec{V}^{\top}$, where the columns $\vec{v}_n$ of $\vec{V}$ are the eigenvectors, interpreted as the Fourier modes, and $\boldsymbol{\Lambda}$ is a diagonal matrix of the real eigenvalues $\lambda_n$, interpreted as the frequencies, or rates of oscillation of the Fourier modes. The \emph{Graph Fourier Transform (GFT)} maps spatial signals $\vec{X}\in\Rb^{N \times d}$ to their frequency coefficients
\begin{equation*}
\mathcal{F}_{\vec{S}}(\vec{X})  \coloneqq \vec{V}^{\top}\vec{X}.
\end{equation*}
The inverse GFT is given by $\vec{V}\vec{Y}$ formula reads $\vec{X}=\vec{V}\vec{V}^{\top}\vec{X}$. 

\section{Derivation of MPNNs and spectral GNNs from 
equivariant machine learning}
\label{app:extensive_discussion_mpnn_spectral}
This appendix provides a more detailed discussion of message passing neural networks (MPNNs) and spectral graph neural networks (GNNs), with a particular focus on their underlying symmetry principles and invariance requirements. We present two distinct conceptual motivations for these two classes of architectures: MPNNs arise naturally from permutation-equivariant processing of neighborhoods, while spectral GNNs are motivated by invariance with respect to the choice of eigenbases of graph operators. In addition, we discuss alternative viewpoints, modeling choices, and definitions that appear in the literature, highlighting both common formulations and less standard extensions. This extended discussion complements the concise definitions given in the main text and clarifies the scope of architectures considered in this work.

\subsection{Definition of MPNNs}
\label{app:mpnns_extensive}
We present one way to conceptually derive MPNNs from symmetry-preserving design principles.
Given a graph $G$ with node features $\vec{X}\in\Rb^{N\times d}$ (assuming $V(G)=[N]$), consider the collection of neighborhoods $\{N_i\}_{i\in[N]}$.
For each neighborhood $N_i$, define the multiset of neighbor features
\begin{equation*}
\vec{Y}_i \coloneq \{\{\vec{X}_j \mid j\in N_i\}\}.
\end{equation*}
The collection $\{(\vec{X}_i,\vec{Y}_i)\}_{i\in[N]}$ can be viewed as a representation of the graph-feature pair $(G,\vec{X})$.
An MPNN layer updates the feature at vertex $i$ as a function of $(\vec{X}_i,\vec{Y}_i)$, shared across vertices.

Two symmetry requirements arise naturally:
(1) the ordering of vertices $[N]$ is arbitrary, hence the overall mapping should be equivariant to permutations of vertices;
(2) within each neighborhood, the ordering of neighbors is arbitrary, hence the neighborhood map should be invariant to permutations of the multiset $\vec{Y}_i$.
Concretely, for any permutation $\sigma\in\mathcal{S}_K$ with $K=|N_i|$, the neighborhood map $\psi$ should satisfy
\[
\psi(\vec{X}_i,\sigma\vec{Y}_i)=\psi(\vec{X}_i,\vec{Y}_i).
\]

A classical construction for permutation-invariant multiset functions is given by the \emph{DeepSets} formulation
\begin{equation}
\label{eq:deepset_app}
\psi(\vec{X}_i, \vec{Y}_i)
=
\phi \left(
\vec{X}_i,
\sum_{\vec{y}\in \vec{Y}_i} \kappa(\vec{y})
\right),
\end{equation}
where $\kappa$ and $\phi$ are suitable neural networks. In fact, it was shown that DeepSets are universal approximators, namely, any continuous permutation invariant function can be approximated by a DeepSet \cite{DBLP:conf/nips/ZaheerKRPSS17}.

This formulation directly leads to the standard class of message passing architectures commonly referred to as sum-aggregation MPNNs, which take the form
\begin{equation}
\label{eq:mpnn_deepsets}
\vec{x}_u^{(t+1)} \coloneqq \phi_t \left( \vec{x}_u^{(t)}, \sum_{v\in N_G(u)} \kappa_t(\vec{x}_v^{(t)}) \right),
\end{equation}
for all $u \in V(G)$ and $t \in [L]$. Here, $\kappa_t$ and $\phi_t$ are learnable, parameterized functions, such as feed-forward neural networks.

Interpreting the sum as an aggregation function suggests a more general update--aggregate template,
\begin{equation}
\label{eq:DS_app}
\psi(\vec{X}_i,\vec{Y}_i)
=
\text{UPD} \left(
\vec{X}_i,
\text{AGG}\bigl(\{\{\vec{y}\mid \vec{y}\in \vec{Y}_i\}\}\bigr)
\right),
\end{equation}
where $\text{AGG}$ may be sum, mean, max, or other permutation-invariant aggregators.

Applying \cref{eq:DS_app} at all nodes yields the standard MPNN update rule: at layer $t$, define $\vec{Y}_i^{(t)}=\{\{\vec{x}_j^{(t)}\mid j\in N_i\}\}$ and set
$\vec{x}_i^{(t+1)}=\psi_t(\vec{x}_i^{(t)},\vec{Y}_i^{(t)})$ for a shared map $\psi_t$.

\subsection{Definition of spectral GNNs}
\label{app:spectral_gnns_extensive}

Let $\vec{S}$ be a self-adjoint graph shift operator associated with $(G,\vec{X})$.
A spectral GNN layer is defined as a function that processes the signal $\vec{X}$ in the spectral domain and maps the result back to the signal domain.
If $\vec{S}=\vec{V}\boldsymbol{\Lambda}\vec{V}^\top$ is an orthonormal eigendecomposition, one may consider the generic form
\begin{equation}
\label{eq:GenSpec_app}
\psi(\vec{S},\vec{V},\vec{X})
\coloneqq
\vec{V}\,\hat{\psi} \left(\vec{V}^\top \vec{X}\right),
\end{equation}
where $\hat{\psi}:\Rb^{N\times d}\to\Rb^{N\times d'}$ acts on the spectral coefficients.

However, \cref{eq:GenSpec_app} is not sufficient to define a graph-consistent operation, because $\vec{S}$ can admit multiple valid eigenbases:
eigenvectors are defined up to sign, and in the presence of repeated eigenvalues, any orthonormal basis of an eigenspace is admissible.
Therefore, for $\psi$ to depend only on $(\vec{S},\vec{X})$ (and not on arbitrary eigendecomposition choices), it should be invariant to the chosen eigenbasis.
Concretely, for any two eigenbases $\vec{V}$ and $\vec{V}'$ of $\vec{S}$, we require
\begin{equation}
\label{eq:EigenbasisInv_app}
\forall\,\vec{X}\in\Rb^{N\times d}:\quad
\psi(\vec{S},\vec{V},\vec{X})
=
\psi(\vec{S},\vec{V}',\vec{X}).
\end{equation}

A classical approach is to restrict $\hat{\psi}$ to be diagonal in the eigenbasis.
In the one-dimensional case ($d=1$), this means
\[
\hat{\psi}(\vec{Y})
\coloneqq
h(\boldsymbol{\Lambda})\vec{Y},
\qquad
h(\boldsymbol{\Lambda})
=
\text{diag} \bigl(h(\lambda_1),\ldots,h(\lambda_N)\bigr),
\]
for a scalar function $h:\Rb\to\Rb$. Then $\psi(\vec{S},\vec{X})=h(\vec{S})\vec{X}=\vec{V}h(\boldsymbol{\Lambda})\vec{V}^\top\vec{X}$.

In the case of higher-dimensional initial node features (often referred to as the multichannel setting in spectral GNNs) this construction naturally extends to
\begin{equation}
\label{eq:SF_app}
\forall \vec{X} \in \Rb^{N\times d}, \quad
\psi(\vec{S},\vec{X})
=
H(\vec{S})\vec{X}
\coloneqq
\sum_{i=1}^N \vec{v}_i \vec{v}_i^{\top}\vec{X}\,H(\lambda_i)^{\top},
\end{equation}
where the frequency response is matrix-valued, $H:\Rb\to\Rb^{d'\times d}$. Such filters are called functional calculus filters since they define an operator-valued function of $\vec{S}$ through its spectrum. Moreover, \cref{eq:SF_app} satisfies the invariance requirement \cref{eq:EigenbasisInv_app}: within each eigenspace, the sum of projectors $\sum_{i:\lambda_i=\lambda}\vec{v}_i\vec{v}_i^\top$ is the orthogonal projector onto that eigenspace and is independent of the particular eigenbasis.

In practice, spectral GNNs are commonly built by stacking such spectral filters with trainable frequency responses, adding a bias, and applying a pointwise nonlinearity, recovering the spectral GNN formulation introduced in \cref{sec:background} via \cref{eq:spec_layer_main}. A frequent choice is to parametrize $H$ as a polynomial, which permits implementing $H(\vec{S})\vec{X}$ directly in the spatial domain as a finite combination of powers of $\vec{S}$. This yields message-passing style implementations as analyzed in \cref{App:P_MPNN}.

\paragraph{Alternative constructions}
There are also alternative approaches to ensure eigenbasis invariance. For example, nonlinear spectral filters (NLSF) \citep{DBLP:conf/nips/LinTL24} are defined by
\begin{equation}
\label{eq:NLSF_app2}
\psi(\vec{X})
\coloneqq
\sum_{n=1}^N
\Big((\vec{v}_n\vec{v}_n^{\top}\vec{X})./|\vec{v}_n^{\top}\vec{X}|\Big)\,
\psi_n \Big(|\vec{v}_j^{\top}\vec{X}|\Big)_{j=1}^N,
\end{equation}
where $\psi_n:\Rb^{N\times d}\to\Rb^{d\times d'}$ is a trainable function, $|\vec{v}_n^{\top}\vec{X}|=(|\vec{v}_n^{\top}\vec{X}_{:,j}|)_{j=1}^d$, and $./$ denotes elementwise division.

Here, $(\vec{v}_n\vec{v}_n^{\top}\vec{X})./|\vec{v}_n^{\top}\vec{X}|$ is a consistent choice of the eigenvectors that depends only on the data, in case all eigenvalues have multiplicity 1. Indeed, $\vec{v}_n\vec{v}_n^{\top}$ is  the eigenspace projection and does not depend on the arbitrary choice of $\vec{v}_n$. Equation (\ref{eq:NLSF_app2}) has an extension for any eigenvalue multiplicities \citep{DBLP:conf/nips/LinTL24}. An additional alternative formulation is given in \citet{DBLP:conf/nips/GeislerKHG24}. Since the diagonal linear filter formulation \cref{eq:SF_app} is the most widespread and commonly taken as the canonical definition of spectral layers, we primarily focus on it in this paper.

\section{Expressivity of MPNNs and spectral GNNs}
\label{app:expressivityofmpnns}

In this section we compare the expressivity of spectral GNNs and message passing neural networks (MPNNs) with sum aggregation. Since message passing architectures with sum aggregation can approximate any continuous permutation-invariant function on neighborhood multisets (of fixed size), this choice entails no loss of generality, and the analysis extends to other aggregation schemes.

We therefore focus on MPNNs of \cref{eq:mpnn_deepsets}, which we repeat here for convenience 
\begin{equation}
\label{eq:GIN}
\vec{x}_i^{(t)} =
\phi^{(t)} \left(
\vec{x}_i^{(t-1)},\;
\sum_{j \in N(i)} \kappa^{(t)}(\vec{x}_j^{(t-1)})
\right),
\end{equation}
where $\vec{x}_i^{(t)} \in \Rb^{d_t}$ denotes the feature vector of node $i$ at layer $t$,
$\kappa^{(t)} \colon \Rb^{d_{t-1}} \to \Rb^{m_t}$ and
$\phi^{t} \colon \Rb^{d_{t-1}} \times \Rb^{m_t} \to \Rb^{d_t}$
are continuous functions, and $N(i)$ denotes the neighborhood of node $i$. Here, $\vec{x}^{(0)}_i \coloneq \vec{x}_i$ are the input features, i.e., $\vec{X}=(\vec{x}^{(0)}_i)_{i=1}^N$, and the MPNN is defined to be $\theta(G,X)=\vec{X}^{(L)}$, where $X^{(L)}=(\vec{x}_i^{(L)})_{i=1}^N$. In graph-level tasks, the output of the MPNN is defined to be some global aggregation of the node features $\vec{X}^{(L)}$, i.e., $\theta(G,X)=\sum_i \vec{x}_i^{(L)}/N$.

We also consider spectral GNNs based on continuous frequency responses (\ref{eq:spec_layer_main}).
For convenience, we repeat the definition: 
\begin{equation}
\label{eq:SF2}
\vec{X}^{(t)} 
= \rho\big(H_t(\vec{S})\vec{X}^{(t-1)} + \vec{b}^{(t)}\vec{1}^\top \big)
\coloneq \rho\Big(
\sum_{i=1}^N \vec{v}_i \vec{v}_i^{\top}\vec{X}^{(t-1)}H_t(\lambda_i)^{\top}
+ \vec{b}^{(t)}\vec{1}^\top
\Big),
\end{equation}
where $H_t:\Rb\rightarrow \Rb^{d_t\times d_{t-1}}$, called the \emph{frequency response}, is continuous, 
$\vec{b}^{(t)} \in \Rb^{d_t}$ is a bias vector,
and $\rho:\Rb\rightarrow\Rb$ is a Lipschitz continuous activation function applied element-wise.

We compare the expressivity of these architectures using two mathematical notions: separation power via graph isomorphism tests, and universal approximation.

\subsection{Polynomial spectral GNNs are MPNNs}
\label{App:P_MPNN}

Consider a spectral GNN based on the adjacency matrix $\vec{S}=\vec{A}$ as the GSO. 
If one restricts the choice of the frequency responses $H=(H_{i,j})$ to be polynomials 
\begin{equation}
\label{eq:polynomial_filter}
H_{i,j}(z)=\sum_{k=0}^K c^{(i,j)}_k z^k,
\end{equation}
then one can implement the linear filter $H(\vec{S})$ of (\ref{eq:SF2}) directly as an MPNN. Namely, for any $i\in[d']$,
\[
\big(H(\vec{S})\vec{X}\big)_i
= \sum_{j=1}^{d} H_{j,i}(\vec{S})\vec{X}_{:,j}
= \sum_{k=0}^K \sum_{j=1}^{d} c^{(j,i)}_k \vec{S}^k \vec{X}_{:,j}.
\]
Hence,
\begin{equation}
H(\vec{S})\vec{X}
= \sum_{k=0}^K \vec{S}^k \vec{X} \vec{C}_k^{\top},
\end{equation}
where $\vec{C}_k = (c^{(i,j)}_k)_{i,j}$.

Including the bias term, a single spectral GNN layer takes the form
\begin{equation}
\label{eq:Pfilt}
\vec{X}^{(t)}
= \rho\Big(
\sum_{k=0}^K \vec{S}^k \vec{X}^{(t-1)} \vec{C}_k^{\top}
+ \vec{b}^{(t)}\vec{1}^\top
\Big).
\end{equation}

The right-hand side of (\ref{eq:Pfilt}) is nothing else but a $K$-layer MPNN since $\vec{S}=\vec{A}$. This construction extends to any \emph{aggregation GSO} $\vec{S}$, as defined in \cref{sec:app_Aggregation_GSO}.

\subsection{Any MPNN can be approximated by polynomial GNN}
\label{app:approx_mlp_mpnns}

We now establish a constructive relation between a single message passing layer whose aggregation and update functions are implemented as ReLU MLPs, and polynomial spectral GNNs with ReLU activation.

We first address a technical issue caused by ReLU nonlinearities, namely preserving access to the initial node features across layers.

\begin{lemma}
\label{lem:carry_initial_node_features}
Let $\rho=$ReLU be applied element-wise and let $\vec{X}\in\Rb^{N\times d}$ be the initial node-feature matrix.
Then there exists a degree--$0$ polynomial spectral GNN layer  that maps $\vec{X}$ to a feature matrix
\[
\vec{Z}\in\Rb^{N\times 2d}
\qquad\text{with}\qquad
\vec{Z}=\big[\rho(\vec{X}),\rho(-\vec{X})\big].
\]
This representation \emph{carries the initial node features} in the following exact sense:
for every matrix $\vec{W}\in\Rb^{d\times r}$, there exists
$\widetilde{\vec{W}}\in\Rb^{2d\times r}$ such that
\begin{equation}
\label{eq:carry_initial_features_linear}
\vec{X}\vec{W} = \vec{Z}\,\widetilde{\vec{W}},
\qquad
\widetilde{\vec{W}}=
\begin{bmatrix}
\vec{W}\\
-\vec{W}
\end{bmatrix}.
\end{equation}
Consequently, although $\vec{X}$ itself is not explicitly present as a channel after ReLU nonlinearities, any subsequent node-wise affine map that depends on the initial node features can be implemented exactly using only $\vec{Z}$.
\end{lemma}

\begin{proof}
Choose a degree--$0$ polynomial spectral layer in \eqref{eq:Pfilt} with output dimension $2d$,
zero bias, and
\[
\vec{C}_0^\top=[\vec{I}_d,-\vec{I}_d],
\]
where $\vec{I}_d$ is the $d$-dimensional identity matrix.
Applying this layer yields
\[
\vec{Z} =\rho\big(\vec{X}[\vec{I}_d,-\vec{I}_d]\big)
= \big[\rho(\vec{X}),\rho(-\vec{X})\big].
\]
Since $\rho(a)-\rho(-a)=a$ for all $a\in\Rb$, we have entrywise
\[
\vec{X}=\rho(\vec{X})-\rho(-\vec{X}).
\]
Multiplying by $\vec{W}$ gives
\[
\vec{X}\vec{W}=\big(\rho(\vec{X})-\rho(-\vec{X})\big)\vec{W}
= \big[\rho(\vec{X}),\rho(-\vec{X})\big]
\begin{bmatrix}
\vec{W}\\
-\vec{W}
\end{bmatrix},
\]
which is exactly \eqref{eq:carry_initial_features_linear}.
\end{proof}

The following lemma shows that polynomial spectral GNNs can carry previous node features to the next iteration in a concatenation form.

\begin{lemma}
\label{lemma:concat_previous}
Let $\rho=$ReLU be applied element-wise. Fix $t\ge 1$ and let $\vec{X}^{(t-1)}\in\Rb^{N\times d_t}$.
For any polynomial filter $\{\vec{C}_k^{(t)}\}_{k=0}^K$ acting as:
$$
\vec{Z}^{(t)} \coloneqq \rho\Big(\sum_{k=0}^K \vec{S}^k \vec{X}^{(t-1)}(\vec{C}_k^{(t)})^\top + \vec{b}^{(t)}\vec{1}^\top\Big)\in\Rb^{N\times m_t},
$$
there exists a polynomial spectral GNN layer whose output can be written as
$$
\vec{X}^{(t)} = \Big[\rho(\vec{X}^{(t-1)}),\,\rho(-\vec{X}^{(t-1)}),\,\vec{Z}^{(t)}\Big]\in\Rb^{N\times (2d_t+m_t)}.
$$
Moreover, this representation carries $\vec{X}^{(t-1)}$ in the sense of \cref{lem:carry_initial_node_features}: for every $\vec{W}\in\Rb^{d_t\times r}$ there exists $\widetilde{\vec{W}}\in\Rb^{(2d_t+m_t)\times r}$ such that
$$
\vec{X}^{(t-1)}\vec{W}=\vec{X}^{(t)}\widetilde{\vec{W}}.
$$
\end{lemma}

\begin{proof}
Use a single layer of the form \eqref{eq:Pfilt} with output dimension $2d_t+m_t$, bias
$\big[\vec{0},\vec{0},\vec{b}^{(t)}\big]$, and coefficients
$$
(\vec{C}_0)^\top=\begin{bmatrix}\vec{I}_{d_t}&-\vec{I}_{d_t}&(\vec{C}_0^{(t)})^\top\end{bmatrix},
\qquad
(\vec{C}_k)^\top=\begin{bmatrix}\vec{0}&\vec{0}&(\vec{C}_k^{(t)})^\top\end{bmatrix}\ \ (k\ge1).
$$
Applying ReLU yields the stated concatenation. The final claim follows by applying \cref{lem:carry_initial_node_features} to the first $2d_t$ channels and setting the remaining $m_t$ rows of $\widetilde{\vec{W}}$ to zero.
\end{proof}

\begin{lemma}
\label{lem:one_layer_MPNN_to_spectral}
Consider a graph with adjacency matrix $\vec{A}$, and  a single MPNN layer
\begin{equation}
\label{eq:mpnn_one_layer}
\vec{x}_i^{(1)}=\phi \left(\vec{x}_i,\sum_{j\in N(i)} \kappa(\vec{x}_j)\right),
\end{equation}
where $\kappa \colon \Rb^{d}\to\Rb^{m}$ and $\phi\colon \Rb^{d}\times\Rb^{m}\to\Rb^{d'}$ are ReLU MLPs, i.e., $\rho=$ReLU.
Let $\mathrm{depth}(\kappa)$ and $\mathrm{depth}(\phi)$ denote the number of layers in the MLPs $\kappa$ and $\phi$, respectively. Then, there exists a ReLU spectral GNN with GSO $\vec{S}=\vec{A}$ and polynomial filters of degree at most $1$ that exactly implement \cref{eq:mpnn_one_layer}. That is, there exists a finite number of layers $L_{\text{spec}}\in \Nb$ such that:
$$\vec{X}^{(L_{\text{spec}})}_{i,:}=\vec{x}_i^{(1)},$$
for all nodes $i$. Specifically we show that $L_{\text{spec}} = \text{depth}(\kappa)+1+\text{depth}(\phi)$.
\end{lemma}

\begin{proof}
A spectral GNN layer with polynomial degree at most $1$ has the form
\begin{equation}
\label{eq:spec_layer_degree1}
\vec{X}^{(t+1)}=\rho\Big(\vec{X}^{(t)}\vec{C}_{0,t}^\top+\vec{A}\vec{X}^{(t)}\vec{C}_{1,t}^\top
+\vec{b}_t\vec{1}^\top\Big).
\end{equation}

\medskip
\noindent
If we set $\vec{C}_{1,t}=\vec{0}$ in \cref{eq:spec_layer_degree1}, the layer reduces to
\[
\vec{X}^{(t+1)}=\rho\big(\vec{X}^{(t)}\vec{C}_{0,t}^\top+\vec{b}_t\vec{1}^\top\big),
\]
which applies an affine map plus bias followed by $\rho$ independently at each node.
Thus, for any MLP with activation $\rho$ and $\ell$ affine layers, there are $\ell$ such degree--$0$ spectral layers
realizing it node-wise.
Applying this to $\kappa$, which has $\text{depth}(\kappa)$ affine layers, yields after $\text{depth}(\kappa)$ spectral layers
a matrix $\vec{U}\in\Rb^{N\times m}$ with $\vec{U}_{i,:}=\kappa(\vec{x}_i)$ for all $i$.

\medskip
\noindent
Choose one layer with $\vec{C}_{0,t}=\vec{0}$, $\vec{C}_{1,t}=\vec{I}_m$, and $\vec{b}_t=\vec{0}$, so that
\[
\vec{U}'=\vec{A}\vec{U}.
\]
Then for each node $i$,
\[
\vec{U}'_{i,:}=\sum_{j=1}^N \vec{A}_{ij}\vec{U}_{j,:}=\sum_{j\in N(i)}\kappa(\vec{x}_j),
\]
which matches the aggregated term in \cref{eq:mpnn_one_layer}.

\medskip
\noindent
To feed both $x_i$ and $\vec{U}'_{i,:}$ into $\phi$, we work with an augmented feature matrix. This can be done due to \cref{lemma:concat_previous}.

\[
\vec{Y} \coloneq \big[\vec{X} \vec{U}'\big]\in\Rb^{N\times(d+m)}.
\]
By choosing block-structured matrices $\vec{C}_{0,t}$ (and biases $\vec{b}_t$) in degree--$0$ layers, we may (i) carry the $\vec{X}$-coordinates forward unchanged and (ii) apply the affine maps of $\phi$ to the full concatenated vector. Since $\phi$ has $\mathrm{depth}(\phi)$ layers, it can be realized node-wise by $\mathrm{depth}(\phi)$ additional degree--$0$ spectral layers, producing $\vec{X}^{(L_{\text{spec}})}\in\Rb^{N\times d'}$ such that for every $i$,
\[
\vec{X}^{(L_{\text{spec}})}_{i,:}=\phi\big(\vec{x}_i,\vec{U}'_{i,:}\big)
=\phi\left(\vec{x}_i,\sum_{j\in N(i)}\kappa(\vec{x}_j)\right).
\]

\medskip
\noindent
The construction uses $\text{depth}(\kappa)$ degree--$0$ layers to implement $\kappa$,
$1$ degree--$1$ layer to aggregate with $\vec{A}$, and $\text{depth}(\phi)$ degree--$0$ layers to implement $\phi$. Hence $L_{\text{spec}}=\text{depth}(\kappa)+1+\text{depth}(\phi)$.
\end{proof}

\begin{remark}
If one allows spectral GNN architectures to use general activation functions—so that different layers and channels may employ different nonlinearities—then the implementation of an MPNN via a polynomial GNN becomes significantly simpler.
\end{remark}

Next, we show that any MPNN with sum aggregation and general continuous message and update functions can be uniformly approximated by a polynomial GNN.

\begin{theorem}[Uniform approximation of MPNNs by spectral GNNs]
\label{thm:spec_approx_mpnn_uniform}
Fix integers $L\in\mathbb{N}$ and $d_{\max}\in\mathbb{N}$, and let $\mathcal{G}_{S_0,d_{\max}}$ denote the class of finite graphs
with maximum degree at most $d_{\max}$ and initial node features satisfying $\vec{x}_i^{(0)} \in [-S_0,S_0]^{d_0}$ for all nodes $i$.

Consider an $L$-layer MPNN of the form~\cref{eq:GIN},
\[
\vec{x}_i^{(t)} =
\phi^{(t)} \left(
\vec{x}_i^{(t-1)},\;
\sum_{j \in N(i)} \kappa^{(t)}(\vec{x}_j^{(t-1)})
\right),
\qquad t=1,\dots,L,
\]
where each $\kappa^{(t)} : \Rb^{d_{t-1}} \to \Rb^{m_t}$ and
$\phi^{(t)} : \Rb^{d_{t-1}} \times \Rb^{m_t} \to \Rb^{d_t}$ are continuous.

Then for every $\varepsilon>0$ there exists a spectral GNN of the form~\cref{eq:SF2} with GSO $\vec{S}=\vec{A}$ and $\widetilde{L} \in \Nb$ layers, whose frequency responses are polynomial, such that for every graph $G\in\mathcal{G}_{S_0,d_{\max}}$, 
\[
\max_{i}\|\vec{x}_i^{(L)}-\widetilde{\vec{x}}_i^{(\widetilde{L})}\|_2 < \varepsilon,
\]
where $\widetilde{\vec{x}}_i^{(\widetilde{L})}$ are the output node features of the spectral GNN  on $G$, and $\vec{x}_i^{(L)}$ are the output node features of the MPNN on $G$.
\end{theorem}

\begin{proof}

\smallskip
\noindent\textbf{Boundedness of intermediate MPNN features}
Set $R_0\coloneqq \sqrt{d_0}\,S_0$, so that $\|\vec{x}_i^{(0)}\|_2\le R_0$ for all nodes $i$.
We claim that there exist finite constants $R_t<\infty$,
depending only on $S_0$, $d_{\max}$, and the maps
$\{\phi^{(s)},\kappa^{(s)}\}_{s\le t}$, such that for every graph
$G\in\mathcal{G}_{S_0,d_{\max}}$,
\begin{equation}
\label{eq:Rt_bound}
\max_i \|x_i^{(t)}\|_2 \le R_t,
\qquad t=0,1,\dots,L.
\end{equation}

The claim is proved by induction.
Assume \cref{eq:Rt_bound} holds at layer $t-1$.
Since $\kappa^{(t)}$ is continuous, it is bounded on the compact ball
$\mathcal K_t\coloneqq\{\vec{u} \mid \|\vec{u}\|_2\le 1+R_{t-1}\}$; define
\[
B_t \coloneqq \sup_{u\in\mathcal K_t}\|\kappa^{(t)}(u)\|_2 < \infty.
\]
For any node $i$, using $| N(i)|\le d_{\max}$,
\[
\Bigl\|\sum_{j\in N(i)} \kappa^{(t)}(\vec{x}_j^{(t-1)})\Bigr\|_2
\le d_{\max} B_t.
\]
Hence $\phi^{(t)}$ is evaluated only on the compact set
\[
\mathcal K_t\times \mathcal S_t,
\qquad
\mathcal S_t\coloneqq\{\vec{s} \mid \|\vec{s}\|_2\le 1+ d_{\max}B_t\},
\]
on which it is bounded. This yields a finite $R_t$ and completes the induction.

\smallskip
\noindent\textbf{Implementation of the universal approximation theorem}
Fix a layer $t\in\{1,\dots,L\}$.
By the universal approximation theorem, for any $\eta>0$ there exists a ReLu MLP
$\widehat\kappa_{\eta}^{(t)}:\mathbb R^{d_{t-1}}\to\mathbb R^{m_t}$ such that
\begin{equation}
\label{eq:kappa_ua}
\sup_{\vec{u}\in\mathcal K_t}
\|\kappa^{(t)}(\vec{u})-\widehat\kappa_{\eta}^{(t)}(\vec{u})\|_2 \le \eta.
\end{equation}
Similarly, since $\mathcal K_t\times\mathcal S_t$ is compact, there exists an MLP
$\widehat\phi^{(t)}$ satisfying
\begin{equation}
\label{eq:phi_ua}
\sup_{(\vec{u},\vec{s})\in\mathcal K_t\times\mathcal S_t}
\|\phi^{(t)}(\vec{u},\vec{s})-\widehat{\phi}_{\eta}^{(t)}(\vec{u},\vec{s})\|_2 \le \eta.
\end{equation}

\smallskip
\noindent\textbf{Realization using a polynomial spectral GNN}
Let $\hat{\vec{x}}_i^{(t)}$ denote the MPNN outputs induced by replacing $\kappa^{(t)}$ and $\phi^{(t)}$ with $\hat{\kappa}_{\eta}^{(t)}$ and $\hat{\phi}_{\eta}^{(t)}$, respectively. Then by \cref{lem:one_layer_MPNN_to_spectral} we can construct a spectral $\widetilde{L}$-Layer GNN with GSO being the adjacency matrix and polynomial filters, and show that $\hat{\vec{x}}_i^{(L)} = \vec{X}^{(\widetilde{L})}_i$, for all nodes $i$. Here, $\vec{X}^{(\widetilde{L})}_i$ represents the node feature of node $i$ computed after $\widetilde{L}$ layers of the spectral GNN.

\smallskip
\noindent\textbf{Error propagation}
Let $\hat{\vec{x}}_i^{(t)}$ denote the node features produced by the approximating
network, and define the layerwise error
\[
e_t \coloneqq \max_i \|\vec{x}_i^{(t)}-\hat{\vec{x}}_i^{(t)}\|_2,
\qquad t \in [L],
\]
and $e_0=0$.

Since $\kappa^{(t)}$ is continuous on the compact set  $\mathcal K_t$, it is uniformly continuous there. Thus, by the characterization of uniform continuity (\cref{remark:unif_cont_char}) of $\kappa^{(t)}$ on the compact set $\mathcal{K}_t$ there exists a nondecreasing function $\omega_{\kappa,t}$ such that

$$
\|\kappa^{(t)}(\vec{u})-\kappa^{(t)}(\vec{v})\|_2
\le \omega_{\kappa,t}(\|\vec{u}-\vec{v}\|_2),
\qquad \vec{u},\vec{v}\in\mathcal{K}_t,
$$
with $\omega_{\kappa,t}(r)\to 0$ as $r\to 0$.
An analogous statement holds for $\phi^{(t)}$ on
$\mathcal K_t\times\mathcal S_t$, with $\omega_{\phi,t}$.

\smallskip
\noindent
\textbf{Bounding the message error}
Define the true and approximate messages
\[
\vec{m}_i^{(t)} \coloneqq \sum_{j\in N(i)} \kappa^{(t)}(\vec{x}_j^{(t-1)}),
\qquad
\hat{\vec{m}}_{{\eta},i}^{(t)} \coloneqq \sum_{j\in N(i)}
\hat\kappa_{\eta}^{(t)}(\hat{\vec{x}}_j^{(t-1)}).
\]

Then, for each node $i$,
\begin{align}
\|\vec{m}_i^{(t)}-\hat{\vec{m}}_{\eta,i}^{(t)}\|_2
&=
\Bigl\|
\sum_{j\in N(i)} \kappa^{(t)}(\vec{x}_j^{(t-1)})
-
\sum_{j\in N(i)} \hat\kappa_{\eta}^{(t)}(\hat{\vec{x}}_j^{(t-1)})
\Bigr\|_2 \notag\\
&\le
\sum_{j\in N(i)}
\bigl\|
\kappa^{(t)}(\vec{x}_j^{(t-1)})
-
\hat\kappa_{\eta}^{(t)}(\hat{\vec{x}}_j^{(t-1)})
\bigr\|_2 \notag\\
&\le
\sum_{j\in N(i)}
\Bigl(
\bigl\|
\kappa^{(t)}(\vec{x}_j^{(t-1)})
-
\kappa^{(t)}(\hat{\vec{x}}_j^{(t-1)})
\bigr\|_2
+
\bigl\|
\kappa^{(t)}(\hat{\vec{x}}_j^{(t-1)})
-
\hat\kappa_{\eta}^{(t)}(\hat{\vec{x}}_j^{(t-1)})
\bigr\|_2
\Bigr) \notag\\
&\le
\sum_{j\in N(i)}
\Bigl(
\omega_{\kappa,t}(e_{t-1})+\eta
\Bigr) \notag\\
&=
|N(i)|\,\omega_{\kappa,t}(e_{t-1})+|N(i)|\,\eta \notag\\
&\le
d_{\max}\,\omega_{\kappa,t}(e_{t-1})+d_{\max}\,\eta,
\label{eq:message_diff}
\end{align}
where we used the definition of $e_{t-1}$ (as the maximum difference), the uniform continuity property for the nondecreasing function
$\omega_{\kappa,t}$ of $\kappa^{(t)}$ on $\mathcal K_t$, the uniform
approximation bound \cref{eq:kappa_ua}, and the degree bound
$|N(i)|\le d_{\max}$.

\smallskip
\noindent
We now have
\begin{align*}
\|\vec{x}_i^{(t)}-\hat{\vec{x}}_i^{(t)}\|_2
&\le
\|\phi^{(t)}(\vec{x}_i^{(t-1)},\vec{m}_i^{(t)})
-\phi^{(t)}(\hat{\vec{x}}_i^{(t-1)},\hat{\vec{m}}_i^{(t)})\|_2 \\
&\quad+
\|\phi^{(t)}(\hat{\vec{x}}_i^{(t-1)},\hat{\vec{m}}_i^{(t)})
-\hat{\phi}_{\eta}^{(t)}(\hat{\vec{x}}_i^{(t-1)},\hat{\vec{m}}_{\eta,i}^{(t)})\|_2.
\end{align*}
Using uniform continuity of $\phi^{(t)}$, the approximation
bound \cref{eq:phi_ua}, and \cref{eq:message_diff} we obtain
\[
\|\vec{x}_i^{(t)}-\hat{\vec{x}}_i^{(t)}\|_2
\le
\omega_{\phi,t} \Bigl(
e_{t-1}
+
d_{\max}\,\omega_{\kappa,t}(e_{t-1})
+
d_{\max}\eta
\Bigr)
+
\eta.
\]
Taking the maximum over $i$ yields the recursion
\[
e_t
\le
\omega_{\phi,t} \Bigl(
e_{t-1}
+
d_{\max}\,\omega_{\kappa,t}(e_{t-1})
+
d_{\max}\eta
\Bigr)
+
\eta.
\]

\smallskip
\noindent
Finally, since $L$ is fixed and all domains are compact, the finite families
$\{\kappa^{(t)}\}_{t=1}^L$ and $\{\phi^{(t)}\}_{t=1}^L$ are equicontinuous. That is, there exist functions  $\omega_{\kappa}, \omega_{\phi} \colon [0, \infty) \to [0, \infty)$ such that
$\omega_{\kappa,t}\le\omega_\kappa$ and $\omega_{\phi,t}\le\omega_\phi$
for all $t$. Let
\[
F(r,\eta)\coloneqq
\omega_{\phi} \bigl(r+d_{\max}\omega_{\kappa}(r)+d_{\max}\eta\bigr)+\eta,
\]
and define the sequence
\[
E_0(\eta)=0,
\qquad
E_t(\eta)=F(E_{t-1}(\eta),\eta),
\]
we obtain by induction that $e_t\le E_t(\eta)$ for all $t$.
Since $F(0,0)=0$ and $F$ is continuous, $E_L(\eta)\to 0$ as $\eta\to 0$.
Thus, for sufficiently small $\eta$,
\[
\max_i \|\vec{x}_i^{(L)}-\hat{\vec{x}}_i^{(L)}\|_2 < \varepsilon.
\]
\end{proof}

\subsection{MPNN Expressivity via 1-WL}
\label{app:expressivity_as_graph_seperation}

Graph isomorphism tests aim to decide whether two graphs are isomorphic, i.e., whether there exists a bijection
between their node sets that preserves adjacency.
Such tests typically proceed by iteratively refining node labels based on local graph structure, thereby inducing an equivalence relation on the space of graphs. A central example is the $1$-dimensional Weisfeiler--Lehman ($1$-WL) test (also known as color refinement)
The $1$-dimensional Weisfeiler--Lehman ($1$-WL) test, also known as color refinement, is an iterative graph isomorphism heuristic that assigns and refines node labels based on local neighborhood structure. Starting from an initial labeling $\ell^{(0)}:V(G)\to\Sigma$, where $\Sigma$ is a countable set of labels, the test proceeds in rounds $t\ge 1$ by updating each node label according to
\[
\ell^{(t)}(v)
=
\text{HASH}\!\left(
\ell^{(t-1)}(v),
\{\{\ell^{(t-1)}(u)\mid u\in N_G(v)\}\}
\right),
\]
where $\text{HASH}$ is an injective function on its arguments and $\{\{\cdot\}\}$ denotes a multiset.  
If the graphs are equipped with initial node features, the initialization $\ell^{(0)}$ is taken to be any labeling that is consistent with these features, i.e., nodes with identical features receive the same initial label. Two graphs are declared non-isomorphic if the multisets of node labels differ at any iteration; otherwise, the test cannot distinguish them.

Let $\mathcal{F}$ be a class of GNN architectures, viewed as mappings from graphs to real-valued vectors, and let $\mathcal{T}$ be a graph isomorphism test inducing an equivalence relation $\sim_{\mathcal{T}}$ on graphs. We compare the expressive power of $\mathcal{F}$ and $\mathcal{T}$ as follows:

\begin{itemize}
    \item $\mathcal{F}$ is not less expressive than $\mathcal{T}$ if, for any graphs $G$ and $H$ with
    $G \not\sim_{\mathcal{T}} H$, there exists a GNN $\Phi \in \mathcal{F}$ such that $\Phi(G) \neq \Phi(H)$.
    \item $\mathcal{F}$ is not more expressive than $\mathcal{T}$ if, for any graphs $G$ and $H$ with
    $G \sim_{\mathcal{T}} H$, we have $\Phi(G) = \Phi(H)$ for all $\Phi \in \mathcal{F}$.
    \item $\mathcal{F}$ is as expressive as $\mathcal{T}$ if it is both not less expressive and not more
    expressive than $\mathcal{T}$.
\end{itemize}

Below we state a central theorem connecting the expressive power of MPNNs as as defined in \cref{eq:GIN} with that of the $1$-WL test.

\begin{theorem}[\citep{Mor+2019, Xu+2021}]
\label{thm:mpnns_are_1wl}
Let any $L \in \Nb$ and $\mathcal{F}_{\text{GIN},L}$ denotes the class of MPNNs whose layers are of the form of \cref{eq:GIN}, i.,e.,
\begin{equation*}
\vec{x}_i^{(t)} =\phi^{(t)} \left(\vec{x}_i^{(t-1)}, \sum_{j \in N(i)}\kappa^{(t)}(\vec{x}_j^{(t-1)})\right),
\end{equation*}
followed by a sum-pooling operator over final node feature to produce a graph-level representation. Then $\mathcal{F}_{\text{GIN},L}$ is as expressive as the $1$-dimensional Weisfeiler--Lehman ($1$-WL) graph isomorphism test after $L$ iterations
\end{theorem}

\subsection{Aggregation graph shift operators}
\label{sec:app_Aggregation_GSO}

Since there is no formal mathematical definition that determines which operator constitutes a graph shift operator (GSO) of a given graph, one can construct exotic GSOs that allow spectral GNNs to break any given graph isomorphism test. Such constructions may encode global information directly into the operator, thereby artificially increasing the expressive power of the resulting spectral architecture.

To enable a meaningful comparison of the expressive power of MPNNs and spectral GNNs, we must exclude graph shift operators that explicitly encode global structural information. To this end, we restrict our attention to GSOs that can be computed from the adjacency matrix using only local operations. We refer to such operators as aggregation graph shift operators, defined below.

\aggregationGSO*

One can include an additional constraint on the definition of an aggregation GSO: that only entries such that $A_{i,j}\neq 0$ can be nonzero in the GSO. However, even without this additional restriction, spectral GNNs based on aggregation GSOs are no more expressive than 1-WL, as it will become clear in \cref{sec:app_expressivity_of_spectral_gnns}. Since the additional restriction on the support of the GSO plays no role in the analysis, we omit it from the formal definition of aggregation GSO.

Although this definition may at first appear tailored to align MPNNs and spectral GNNs, we show below that all commonly used GSOs in the literature satisfy it. In particular, standard choices such as the adjacency matrix, the combinatorial and normalized graph Laplacians, as well as Cayley shift operators (via their Jacobi approximations), are all aggregation graph shift operators in
the sense of \cref{def:aggregationGSO}.

\begin{example}
\label{example:aggregation_GSOs}
We give several examples of common graph shift operators that can be implemented by MPNNs with sum aggregation. Note that the signal $\vec{1}$ can always be concatenated to $\vec{f}$ by an MPNN. Indeed, as the first layer, just define $\phi^{(1)}(x,y)=(x,1)$. 

\begin{itemize}
\item
The degree matrix $\vec{D}\coloneq \text{diag}(d_1,\ldots,d_N)$, with $d_i = |N(i)|$, for $i \in [N]$ can be implemented by a three-layer MPNN, as we show next.  
Without loss of generality, assume $x_i^{(0)}=f_i\in\Rb$. 
The mapping $\vec{f}\mapsto\vec{D}\vec{f}$ can be implemented in three layers as follows:
$$
x_i^{(1)}=\phi^{(1)}(f_i,0)=(f_i,1),
$$
$$
x_i^{(2)}=\phi^{(2)} \Bigl((f_i,1),\sum_{j\in N(i)}(0,1)\Bigr)
=\phi^{(2)}\bigl((f_i,1),(0,d_i)\bigr)=(d_i,f_i),
$$
$$
x_i^{(3)}=\phi^{(3)}\bigl((d_i,f_i),0\bigr)=d_i f_i=(\vec{D}\vec{f})_i,
$$

\item
The combinatorial Laplacian $\vec{L}=\vec{D}-\vec{A}$ can be implemented by a four-layer MPNN. Starting from $x_i^{(0)}=f_i$, the first three layers compute in parallel
$$
x_i^{(3)}=\bigl((\vec{D}\vec{f})_i,(\vec{A}\vec{f})_i\bigr),$$
where $(\vec{D}\vec{f})_i=d_i f_i$ is obtained as in the previous example, and
$$
(\vec{A}\vec{f})_i=\sum_{j\in N(i)} f_j
$$
is computed by sum aggregation. The fourth layer applies a   linear update function,
$$
x_i^{(4)}=\phi^{(4)}\bigl((d_i f_i,\,(\vec{A}\vec{f})_i),0\bigr)
=d_i f_i-(\vec{A}\vec{f})_i
=(\vec{L}\vec{f})_i,
$$
thereby implementing the combinatorial Laplacian.

\item
The symmetric normalized Laplacian $\vec{L}_{\mathrm{sym}}=\vec{D}^{-1/2}\vec{L}\vec{D}^{-1/2}$ can be implemented by a finite-depth MPNN. Starting from $x_i^{(0)}=f_i$, the degree signal $d_i$ is first computed as above and used to form
$$
x_i^{(1)}=(d_i^{-1/2}f_i),
$$
via a pointwise transformation. Next, applying the construction for the combinatorial Laplacian yields
$$
x_i^{(2)}=(\vec{L}( \vec{D}^{-1/2}\vec{f}))_i.
$$
Finally, a further pointwise layer multiplies the result by $d_i^{-1/2}$,
$$
x_i^{(3)}=d_i^{-1/2}(\vec{L}(\vec{D}^{-1/2}\vec{f}))_i
=(\vec{L}_{\mathrm{sym}}\vec{f})_i,
$$

\item
Cayley shift operators implemented via finite Jacobi approximations are aggregation graph shift operators.
The Jacobi approximation expresses the Cayley transform as a finite combination of repeated applications of a base graph shift operator, together with pointwise linear transformations. Since each application of the shift operator corresponds to a sum-aggregation message passing layer, the resulting mapping can be implemented by a finite-depth MPNN. Hence, Jacobi approximated Cayley shift operators satisfy \cref{def:aggregationGSO}.

\end{itemize}
\end{example}

\subsection{Expressivity of spectral GNNs with aggregation GSOs}
\label{sec:app_expressivity_of_spectral_gnns}

In this section, we build on the results in \cref{App:P_MPNN}, \cref{app:approx_mlp_mpnns} to analyze the expressive power of spectral GNNs. We show that spectral GNNs with aggregation GSOs and arbitrary continuous spectral filters cannot be more expressive
than the $1$-WL test. Moreover, we show that by using the adjacency matrix as the GSO together with polynomial filters of degree at most $1$, spectral GNNs can exactly match the expressive power of $1$-WL. The latter result follows from the fact that spectral GNNs of this type can uniformly approximate any MPNN of the form defined in \cref{eq:GIN}, combined with the equivalence between MPNNs and $1$-WL stated in \cref{thm:mpnns_are_1wl}.

We begin with the following lemma, which is the key ingredient for establishing an upper bound on the expressive power of spectral GNNs.

\begin{lemma}
\label{lemma:poly_gnn_is_mpnn}
Spectral GNNs with polynomial filters based on aggregation GSOs are special cases of message passing neural networks
with sum aggregation.
\end{lemma}

\begin{proof}
Let $\vec{S}$ be an aggregation GSO.
By the definition of aggregation GSO (see \cref{def:aggregationGSO}), there exists a finite-depth message passing neural network with sum aggregation, denoted by $\mathrm{MPNN}_{\vec{S}}$, such that for every node signal $\vec{f}\in\Rb^N$,
\begin{equation*}
(\mathrm{MPNN}_{\vec{S}}(\vec{f}))_i = (\vec{S}\vec{f})_i,
\end{equation*}
where $(\mathrm{MPNN}_{\vec{S}}(\vec{f}))_i$ denotes the output of $\mathrm{MPNN}_{\vec{S}}$ at node $i$.
Let $T_{\vec{S}}\in\mathbb{N}$ denote the (finite) depth of $\mathrm{MPNN}_{\vec{S}}$.

Consider a polynomial spectral layer of degree $K\in\mathbb{N}$, omitting bias terms for simplicity:
\begin{equation*}
\vec{h} = \rho\bigl(p(\vec{S})\vec{f}\bigr),
\qquad
p(\vec{S}) = \sum_{k=0}^K a_k \vec{S}^k,
\end{equation*}
where $\rho:\Rb\to\Rb$ is a pointwise nonlinearity.

Define $\vec{z}^{(0)}\coloneq\vec{f}$ and recursively
\begin{equation*}
\vec{z}^{(k)} \coloneq \mathrm{MPNN}_{\vec{S}}(\vec{z}^{(k-1)}),
\qquad
k=1,\dots,K.
\end{equation*}
By induction, we have $\vec{z}^{(k)} = \vec{S}^k \vec{f}$ for all $k\le K$.
This construction requires finite depth $K T_{\vec{S}}$.

Using feature concatenation, which can be realized by choosing update maps that ignore the aggregated argument, we maintain the channels
\begin{equation*}
(\vec{z}^{(0)},\dots,\vec{z}^{(K)}).
\end{equation*}
We then apply a pointwise linear map to compute
\begin{equation*}
\vec{u} \coloneq \sum_{k=0}^K a_k \vec{z}^{(k)}= \sum_{k=0}^K a_k \vec{S}^k \vec{f} = p(\vec{S})\vec{f}.
\end{equation*}

Finally, we apply the nonlinearity $\rho$ by choosing the update map
\begin{equation*}
\phi(u,y) = \rho(u),
\end{equation*}
so that the output satisfies $\vec{h} = \rho(\vec{u}) = \rho(p(\vec{S})\vec{f})$.
Hence, the polynomial spectral layer can be implemented by a finite-depth MPNN with sum aggregation.
\end{proof}

We are now ready to characterize the expressive power of spectral GNNs with aggregation GSOs and polynomial filters.

\begin{proposition}
\label{prop:expressivityofspectralgnns}
Let $\mathcal{F}_{\mathrm{spec}}(\vec{S})$ denote the class of spectral GNNs whose layers are of the form
\cref{eq:Pfilt}, i.e.,
\begin{equation*}
\vec{X}^{(t)}
= \rho\Big(
\sum_{k=0}^K \vec{S}^k \vec{X}^{(t-1)} \vec{C}_k^{\top}
+ \vec{b}^{(t)}\vec{1}^\top
\Big),
\end{equation*}
where $\vec{S}$ is the graph shift operator and $\vec{C}_k$ are learnable weight matrices defining a
matrix-valued polynomial filter, followed by a sum-pooling operator over node features to produce a graph-level representation. Then the following statements hold:
\begin{itemize}
\item If $\vec{S}$ is an aggregation graph shift operator, then $\mathcal{F}_{\mathrm{spec}}(\vec{S})$ is not more
expressive than the $1$-dimensional Weisfeiler--Lehman ($1$-WL) test.
\item If $\vec{S}=\vec{A}$ is the adjacency matrix, then $\mathcal{F}_{\mathrm{spec}}(\vec{A})$ is as expressive as the
$1$-dimensional Weisfeiler--Lehman ($1$-WL) test.
\end{itemize}
\end{proposition}

\begin{proof} Let $\vec{S}$ be an aggregation graph shift operator and let $\Phi\in\mathcal{F}_{\mathrm{spec}}(\vec{S})$. By \cref{lemma:poly_gnn_is_mpnn}, $\Phi$ is a special case of an MPNN with sum aggregation. By \cref{thm:mpnns_are_1wl}, any such MPNN is not more expressive than $1$-WL (for a proper number of iterations).
Hence, for any graphs $G,H$ with $G\sim_{1\text{-WL}}H$, we have $\Phi(G)=\Phi(H)$.
Since $\Phi\in\mathcal{F}_{\mathrm{spec}}(\vec{S})$ was arbitrary, this shows that
$\mathcal{F}_{\mathrm{spec}}(\vec{S})$ is not more expressive than $1$-WL.

\medskip
\noindent
By applying the previous argument to $\vec{S}=\vec{A}$, we obtain that $\mathcal{F}_{\mathrm{spec}}(\vec{A})$ is not more expressive than $1$-WL. It remains to prove that $\mathcal{F}_{\mathrm{spec}}(\vec{A})$ is not less expressive than $1$-WL. Let $G,H$ be graphs such that $G\not\sim_{1\text{-WL}}H$. Then there exists $L\in\mathbb{N}$ such that $1$-WL distinguishes $G$ and $H$ after $L$ iterations. By \cref{thm:mpnns_are_1wl}, there exists an $L$-layer MPNN $\Psi$ of the form \cref{eq:GIN} such that
$\Psi(G)\neq\Psi(H)$.

Now fix $\varepsilon>0$ such that
\begin{equation*}
\|\Psi(G)-\Psi(H)\|_2 > 2\varepsilon.
\end{equation*}
Apply \cref{thm:spec_approx_mpnn_uniform} to $\Psi$ on a graph class containing $G$ and $H$. For example by choosing $d_{\max} > \max_{u\in V(G) \cup V(H)}\{|N(u)|\}$ and $S_0$ large enough.
Then there exists a polynomial spectral GNN $\widetilde{\Psi}\in\mathcal{F}_{\mathrm{spec}}(\vec{A})$ such that
\begin{equation*}
\|\Psi(G)-\widetilde{\Psi}(G)\|_2 < \varepsilon
\qquad\text{and}\qquad
\|\Psi(H)-\widetilde{\Psi}(H)\|_2 < \varepsilon.
\end{equation*}
By the triangle inequality,
\begin{equation*}
\|\widetilde{\Psi}(G)-\widetilde{\Psi}(H)\|_2
\ge \|\Psi(G)-\Psi(H)\|_2 - \|\Psi(G)-\widetilde{\Psi}(G)\|_2 - \|\Psi(H)-\widetilde{\Psi}(H)\|_2
> 0,
\end{equation*}
and therefore $\widetilde{\Psi}(G)\neq \widetilde{\Psi}(H)$.
This proves that $\mathcal{F}_{\mathrm{spec}}(\vec{A})$ is not less expressive than $1$-WL.
\end{proof}

We now extend the previous result from polynomial spectral filters to general continuous frequency responses. This follows from the fact that continuous  filters can be uniformly approximated by polynomial ones on the spectrum of the two graph shift operators.

\expressivePower*

We now formalize the approximation argument underlying the proof of \cref{cor:functional_calculus_expressivity}. The key idea is that, since spectral GNNs apply continuous frequency responses to the spectrum of the graph shift operator, each such layer can be uniformly approximated by a polynomial spectral filter on the spectrum of $\vec{S}$. We first show that continuous filters admit uniform polynomial approximations on the spectrum, then we control how these perturbations propagate through the nonlinear layers of a spectral GNN. Finally, we state the main theorem characterizing the expressive power of spectral GNNs with aggregation GSOs and arbitrary continuous filters. We begin with the following three important lemmas.

\begin{lemma}
\label{lem:poly_approx_spectrum}
Let $\vec{S}\in\Rb^{N\times N}$ be diagonalizable with real spectrum $\sigma(\vec{S})=\{\lambda_1,\dots,\lambda_N\}\subseteq [a,b]$. Let $H\colon [a,b]\to\Rb^{d'\times d}$ be continuous. Then for every $\varepsilon>0$ there exists a matrix-valued polynomial $p(z)=\sum_{k=0}^K z^k \vec{C}_k$ such that
\begin{equation*}
\sup_{z\in[a,b]} \|H(z)-p(z)\|_2 < \varepsilon.
\end{equation*}
In particular,
\begin{equation*}
\max_{i\in[N]}\|H(\lambda_i)-p(\lambda_i)\|_2 < \varepsilon.
\end{equation*}
\end{lemma}

\begin{proof}
By the Stone--Weierstrass theorem, each scalar-valued continuous function on $[a,b]$ can be uniformly approximated by polynomials. Apply this entrywise to the matrix-valued function $H$ to obtain a matrix-valued polynomial $p$ such that $\sup_{z\in[a,b]}\|H(z)-p(z)\|_{\infty}<\varepsilon$. Since all norms on a finite-dimensional space are equivalent, we may replace $\|\cdot\|_{\infty}$ by the $2$-norm, up to adjusting $\varepsilon$ by a constant factor depending only on $d$ and $d'$. This yields the claimed bound in operator norm.
\end{proof}

\begin{lemma}
\label{lem:filter_to_operator_bound}
Let $\vec{S}$ be symmetric with eigendecomposition $\vec{S}=\vec{V}\mathrm{diag}(\lambda_1,\dots,\lambda_N)\vec{V}^\top$.
Let $H,p:\Rb\to\Rb^{d'\times d}$ be such that
\begin{equation*}
\max_{i\in[N]}\|H(\lambda_i)-p(\lambda_i)\|_2 \le \varepsilon.
\end{equation*}
Then for matrix $\vec{X}\in\Rb^{N\times d}$,
\begin{equation*}
\|H(\vec{S})\vec{X}-p(\vec{S})\vec{X}\|_F \le \varepsilon \|\vec{X}\|_F.
\end{equation*}
\end{lemma}

\begin{proof}
We have
\begin{equation*}
H(\vec{S})\vec{X} = \sum_{i=1}^N \vec{v}_i \vec{v}_i^\top \vec{X} \, H(\lambda_i)^\top,
\qquad
p(\vec{S})\vec{X} = \sum_{i=1}^N \vec{v}_i \vec{v}_i^\top \vec{X} \, p(\lambda_i)^\top,
\end{equation*}
and thus
\begin{equation*}
\big(H(\vec{S})-p(\vec{S})\big)\vec{X}
=
\sum_{i=1}^N \vec{v}_i \vec{v}_i^\top \vec{X}\, \big(H(\lambda_i)-p(\lambda_i)\big)^\top.
\end{equation*}
Since the projectors $\vec{v}_i\vec{v}_i^\top$ are orthogonal and sum to the identity, we obtain
\begin{equation*}
\| \big(H(\vec{S})-p(\vec{S})\big)\vec{X}\|_F^2
=
\sum_{i=1}^N \|\vec{v}_i \vec{v}_i^\top \vec{X}\, \big(H(\lambda_i)-p(\lambda_i)\big)^\top\|_F^2.
\end{equation*}
Using $\|AB\|_F \le \|A\|_F \|B\|_{\mathrm{op}}$ and the assumption gives
\begin{equation*}
\| \big(H(\vec{S})-p(\vec{S})\big)\vec{X}\|_F^2
\le
\varepsilon^2 \sum_{i=1}^N \|\vec{v}_i \vec{v}_i^\top \vec{X}\|_F^2
=
\varepsilon^2 \|\vec{X}\|_F^2,
\end{equation*}
which proves the claim.
\end{proof}

\begin{lemma}
\label{lem:layer_bound}
Let $\rho$ be $L_\rho$-Lipschitz.
Consider two spectral layers with the same $\vec{S}$, weights, and biases, but different frequency responses $H$ and $p$:
\begin{equation*}
\vec{Y} = \rho\big(H(\vec{S})\vec{X}+\vec{b}\vec{1}^\top\big),
\qquad
\widetilde{\vec{Y}} = \rho\big(p(\vec{S})\vec{X}+\vec{b}\vec{1}^\top\big).
\end{equation*}
Then
\begin{equation*}
\|\vec{Y}-\widetilde{\vec{Y}}\|_F
\le
L_\rho \| \big(H(\vec{S})-p(\vec{S})\big)\vec{X}\|_F.
\end{equation*}
In particular, under the assumptions of \cref{lem:filter_to_operator_bound},
\begin{equation*}
\|\vec{Y}-\widetilde{\vec{Y}}\|_F \le L_\rho \varepsilon \|\vec{X}\|_F.
\end{equation*}
\end{lemma}

\begin{proof}
Since $\rho$ is applied element-wise and is $L_\rho$-Lipschitz, it is $L_\rho$-Lipschitz with respect to the Frobenius norm:
\begin{equation*}
\|\rho(\vec{U})-\rho(\vec{V})\|_F \le L_\rho \|\vec{U}-\vec{V}\|_F.
\end{equation*}
Apply this with $\vec{U}=H(\vec{S})\vec{X}+\vec{b}\vec{1}^\top$ and
$\vec{V}=p(\vec{S})\vec{X}+\vec{b}\vec{1}^\top$, completes the proof.
\end{proof}

We are now ready to prove \cref{cor:functional_calculus_expressivity}. The previous lemmas show that functional-calculus spectral] GNNs with continuous filters can be uniformly approximated by polynomial spectral GNNs, with controlled error across layers.

\begin{proof}[Proof of \cref{cor:functional_calculus_expressivity}]
Fix $\Phi\in \mathcal{F}_{\text{C}}(\vec{S})$ with $T$ layers.
For each layer $t\in[T]$, the spectrum $\sigma(\vec{S})$ is finite and contained in an interval $[a,b]$.
By \cref{lem:poly_approx_spectrum}, for every $\delta>0$ there exists a matrix-valued polynomial $p^{(t)}$ such that
\begin{equation*}
\max_{\lambda\in\sigma(\vec{S})}\|H_t(\lambda)-p^{(t)}(\lambda)\|_2 \le \delta.
\end{equation*}
Let $\widetilde{\Phi}_\delta$ be the polynomial spectral GNN obtained by replacing each $H_t$ by $p^{(t)}$
(and keeping the same biases and activation).

By \cref{lem:layer_bound}, each layer output of $\widetilde{\Phi}_\delta$
approximates the corresponding layer output of $\Phi$, and hence (by iterating the bound through the finite depth $T$)
we obtain that for every graph $G$,
\begin{equation*}
\|\Phi(G)-\widetilde{\Phi}_\delta(G)\|_F \to 0
\qquad\text{as}\qquad \delta\to 0.
\end{equation*}
In particular, for any fixed graphs $G$ and $H$, and for every $\varepsilon>0$, there exist $\delta(\varepsilon)>0$ small enough so that
\begin{equation*}
\|\Phi(G)-\widetilde{\Phi}_{\delta(\varepsilon)}(G)\|_F < \varepsilon
\qquad\text{and}\qquad
\|\Phi(H)-\widetilde{\Phi}_{\delta(\varepsilon)}(H)\|_F < \varepsilon.
\end{equation*}

Now assume that $G\sim_{1\text{-WL}} H$.
Since $\widetilde{\Phi}_\delta$ is a polynomial spectral GNN based on $\vec{S}$, \cref{lemma:poly_gnn_is_mpnn} implies
that $\widetilde{\Phi}_{\delta(\varepsilon)}$ is a sum-aggregation MPNN.
Therefore, by \cref{thm:mpnns_are_1wl}, $\widetilde{\Phi}_{\delta(\varepsilon)}(G)=\widetilde{\Phi}_{\delta(\varepsilon)}(H)$.

Using the triangle inequality yields
\begin{equation*}
\|\Phi(G)-\Phi(H)\|_F
\le
\|\Phi(G)-\widetilde{\Phi}_{\delta(\varepsilon)}(G)\|_F
+
\|\widetilde{\Phi}_{\delta(\varepsilon)}(H)-\Phi(H)\|_F
< 2\varepsilon
\end{equation*}
Therefore, letting $\varepsilon \rightarrow 0$, gives $\Phi(G)=\Phi(H)$. Since $\Phi\in\mathcal{F}_{\text{C}}(\vec{S})$ was arbitrary, this proves that $\mathcal{F}_{\text{C}}(\vec{S})$ is not more expressive than $1$-WL. The equivalency to $1$-WL, follows directly from \cref{prop:expressivityofspectralgnns}, using the adjacency as the GSO and polynomial filters.
\end{proof}

\subsection{Expressivity via universal approximation}
\label{app:expressivity_as_approximation}

Expressivity via graph isomorphism test allows choosing a different GNN per pair of graphs. Hence, it does not model the expressivity of a fixed GNN on the whole data distribution. In this section, we propose an alternative definition of expressivity for fixed GNNs on data distributions. For simplicity, we focus on real-valued outputs, noting that the analysis extends directly to vector-valued settings.

\Expressivityunivapprox*

\begin{proposition}
\label{prop:poly_not_universal_for_fc_adj}
Fix $L\ge 1$. There exists a data distribution $Q$ over graphs with initial node features, $(G,\vec{X})$, such that
the family of $L$-layer polynomial-filter spectral GNNs with sum-pooling readout layer is not a universal approximator of the family of $L$-layer functional calculus spectral GNNs with sum-pooling readout layer on $Q$.
\end{proposition}

\begin{proof}
We construct a data distribution $Q$ and an $L$-layer functional-calculus spectral GNN whose
\emph{sum-pooled} input--output map cannot be uniformly approximated, in essential supremum over $Q$,
by any $L$-layer polynomial-filter spectral GNN (with ReLU activation and arbitrary biases) equipped with the same sum-pooling readout.

\medskip
\noindent
We begin by specifying the distribution $Q$. For each $n\ge 1$, let $G_n$ be the complete graph on $n+1$ vertices with adjacency matrix $\vec{A}_n$.
Its spectrum contains the eigenvalue $\lambda_n = n$, with associated eigenvector, the all ones vector $\vec{1}_{n+1}\in\Rb^{n+1}$.
Define
$$
\vec{x}_n \coloneqq \frac{1}{n+1}\vec{1}_{n+1}\in\Rb^{n+1},
\qquad\text{so that}\qquad
\vec{1}_{n+1}^\top \vec{x}_n = 1.
$$
We define $Q$ by sampling $n\ge 1$ with probability $\Pr(n)=2^{-n}$, setting $G=G_n$ and $\vec{x}=\vec{x}_n$.
Then, for every $M>0$, $Q$ assigns positive probability to graphs whose adjacency matrices have an eigenvalue
larger than $M$.

\medskip
\noindent
Consider an $L$-layer functional-calculus spectral GNN (as in \cref{eq:spec_layer_main})
with identity activation and zero biases, and with the same frequency response at every layer, $H_t(\lambda)=f(\lambda)=\sin\lambda.$

Since $\vec{x}_n$ is an eigenvector of $\vec{A}_n$ associated with $\lambda_n=n$, we have
\[
f(\vec{A}_n)^L \vec{x}_n = (\sin \lambda_n)^L \vec{x}_n = (\sin n)^L \vec{x}_n.
\]
Define the resulting output after the sum pooling as:
\[
F(G,\vec{x}) \coloneqq \vec{1}^\top f(\vec{A})^L \vec{x},
\]
which is precisely the composition of the $L$ spectral layers with the sum-pooling readout.
On the support of $Q$, i.e., this reduces to
$$
F(G_n,\vec{x}_n)=\vec{1}_{n+1}^\top \big((\sin n)^L \vec{x}_n\big) = (\sin n)^L \vec{1}_{n+1}^\top \vec{x}_n = (\sin n)^L.
$$

\medskip
\noindent
Now fix any $L$-layer polynomial-filter spectral GNN with ReLU activation and arbitrary biases, and equip it
with the same sum-pooling readout. Denote the resulting map by $P$.
When evaluating $P$ on $(G_n,\vec{x}_n)$, the input $\vec{x}_n$ lies entirely in the eigenspace of $\lambda_n=n$.
Moreover, each polynomial spectral layer acts diagonally in the eigenbasis (as in \cref{eq:spec_layer_main}),
so all spectral contributions corresponding to eigenvalues different from $\lambda_n$ vanish on this input.
Consequently, the final sum-pooled output depends only on $\lambda_n$.
That is, there exists a function $\psi:\Rb\to\Rb$ such that
\[
P(G_n,\vec{x}_n)=\psi(\lambda_n)=\psi(n)
\qquad\text{for all }n\ge 1.
\]
Furthermore, since the filters are polynomial and the only nonlinearity is ReLU, the function $\psi$ is
piecewise-polynomial with finitely many pieces.

\medskip
\noindent
Fix $\varepsilon\in(0,\frac12)$.
Let $\psi:\Rb\to\Rb$ be piecewise-polynomial with finitely many pieces. Then there exists $T\in\Rb$
and a polynomial $q$ such that $\psi(x)=q(x)$ for all $x\ge T$. If $q$ is non-constant, then $|q(n)|\to\infty$ as $n\to\infty$, so for some integer $n>T$ we have $|q(n)|\ge 2$.
Since $|(\sin n)^L|\le 1$, it follows that
\[
|(\sin n)^L-\psi(n)|
=
|(\sin n)^L-q(n)|
\ge |q(n)|-1
\ge 1
>\varepsilon.
\]
It remains to examine the case where $q \equiv c$, for some $c \in \Rb$. If $c > 1$, then we define $\varepsilon = c-1>0$ then since $|(\sin n)^L|\le 1$ for all $n$, we immediately have $|(\sin n)^L-c|\ge \varepsilon$ for all $n$. Similarly for $c < -1$

Assume now that $c\in[-1,1]$.
Since $1/(2\pi)$ is irrational, Weyl’s equidistribution law (\citep{kuipers1974uniform})) implies that the sequence
$\{n \bmod 2\pi\}_{n\in\mathbb{N}}$ is dense in $[0,2\pi)$. Because the $\sin$ function is continuous and satisfies $\sin n = \sin(n \bmod 2\pi)$, it follows that the sequence $\{\sin n \mid n\in\mathbb{N}\}$ is dense in $[-1,1]$.

Consequently, the sequence $\{(\sin n)^L\}_{n\in\mathbb{N}}$ is dense in $[-1,1]$ if $L$ is odd,
and dense in $[0,1]$ if $L$ is even. In either case, for any $\varepsilon\in(0,\tfrac12)$ there exists an integer $n>T$ such that
$$
|(\sin n)^L - c| > \varepsilon.
$$
Hence,
$$
|F(G_n,\vec{x}_n)-P(G_n,\vec{x}_n)| = |(\sin n)^L - \psi(n)| > \varepsilon.
$$
\end{proof}

The previous result relies on the presence of arbitrarily large spectral values. When the spectrum of the GSO is uniformly bounded, the situation changes.

\begin{proposition}
\label{prop:fc_approx_by_poly_bounded_spectrum}
Fix $L\ge 1$ and suppose the data distribution $Q$ is supported on pairs $(G,\vec{X})$
whose graph shift operator $\vec{S}$ has all eigenvalues contained in a common bounded interval $[-B,B]$ for some $B<\infty$. Then any $L$-layer functional-calculus spectral GNN with continuous filters and sum-pooling readout layer can be uniformly approximated
on $Q$ by an $L$-layer polynomial-filter spectral GNN with sum-pooling readout layer.
\end{proposition}

\begin{proof}
Fix an $L$-layer functional-calculus spectral GNN with continuous spectral filters and sum-pooling readout.
At each layer $t\in[L]$, the network applies a filter of the form $H_t(\vec{S})$, where
$H_t \colon [-B,B]\to\Rb$ is continuous.
Since all eigenvalues of $\vec{S}$ lie in the common bounded interval $[-B,B]$ on the support of $Q$,
the Weierstrass approximation theorem guarantees that for every $\delta>0$ there exists a polynomial
$p_t$ such that
$$
\sup_{\lambda\in[-B,B]} |H_t(\lambda)-p_t(\lambda)| \le \delta.
$$
which implies that,
$\|H_t(\vec{S})-p_t(\vec{S})\|_2\le \delta$ uniformly over all $(G,\vec{X})$ in the support of $Q$. The proof follows similarly to the proof of \cref{cor:functional_calculus_expressivity}.
\end{proof}

\subsection{Extended Discussion for P4}
\label{app:p4_extended}

\paragraph{Stability and Size Transferability.}
In GNNs, stability refers to the property that small perturbations in the input lead to controlled changes in the output. Such perturbations may affect node features, edge weights, or the graph structure itself. This property is important given the presence of noise and distribution shifts over time. 

In the spectral GNN literature, stability is typically analyzed at the level of graph filters, where operator-theoretic techniques are used to bound the sensitivity of functions of the graph shift operator to perturbations of the underlying graph \citep{levie2019transferability,gama2020stability,levie2021transferability,ruiz2021graphon,cervino2022training}. These analyses usually yield bounds that depend at most linearly on the polynomial filter degree or receptive field size, implying that increasing the interaction range does not necessarily require additional nonlinearities.

In contrast, stability analyses for MPNNs are commonly performed layer-by-layer, tracking how perturbations propagate through repeated neighborhood aggregation and nonlinear transformations \citep{Chu+2022}. As a result, bounds may grow exponentially with depth due to repeated compositions of nonlinear maps. However, this discrepancy largely reflects architectural conventions rather than a fundamental distinction between the two paradigms. For example, if nonlinearities are applied only every $K$ layers, separating aggregation from nonlinear transformation, the resulting stability behavior closely resembles that of spectral GNNs with degree-$K$ polynomial filters. Once depth and nonlinearity placement are normalized, the guarantees become highly similar.

Closely related to stability is \emph{size transferability}, which studies how GNN outputs behave when the number of nodes changes. Since GNNs are parameterized independently of graph size, this property is essential for transferring models across graphs of varying scales.

Spectral GNN analyses typically study transferability through spectral convergence arguments \citep{levie2019transferability,ruiz2020graphon,levie2021transferability}. These works connect discrete graph convolutions with continuous convolution operators through the spectrum of the graph filter \citep{ruiz2021graphon,ruiz2023transferability,wang2023convolutional,wang2024geometric}. MPNN analyses instead often rely on operator-theoretic viewpoints, comparing graphs through suitable graph operators without requiring spectral convergence \citep{le2023limits,keriven2020convergence,velasco2024graph,Levin2025Transferring,maskey2025generalization}. 

At first glance, MPNN analyses appear more general because they naturally accommodate sparse graph limits, whereas spectral approaches often rely on dense graph models. Nevertheless, the assumptions used in both settings---such as Lipschitz continuity and regularity of graph operators---are remarkably similar, and the resulting transferability bounds are often nearly equivalent in spirit.

\paragraph{Generalization.}
In machine learning, generalization refers to how well performance on training data transfers to unseen examples. For MPNNs, generalization analyses typically assume bounded node features \citep{maskey2022generalization}, normalized or degree-controlled aggregation schemes \citep{Rac+2025,maskey2025generalization}, and Lipschitz-continuous message and update functions \citep{Vas+2025}; see \citep{vasileiou2025survey} for an overview. Under these assumptions, MPNNs define hypothesis classes consisting of functions that are Lipschitz continuous with respect to suitable graph metrics. Generalization guarantees are then derived through covering number arguments \citep{Lev+2023,Vas+2025}, approximation-theoretic analyses \citep{maskey2022generalization}, or complexity measures such as VC dimension \citep{morris2023wl} and Rademacher complexity \citep{Gar+2020}.

Spectral GNNs admit closely related analyses. Here, hypotheses are defined through operators of the form $g(\vec{L}_G)$, where $\vec{L}_G$ is the graph Laplacian. Generalization guarantees similarly rely on regularity assumptions, including bounded operator norms \citep{Tan+2023}, regularity of the graph spectrum \citep{Yeh+2021}, and approximation properties of the spectral filter class, such as low-degree polynomial or smooth filters \citep{Wang+2025,wang2025generalization,wang2025generalizationrobust}. Thus, despite the different mathematical language, both paradigms ultimately control generalization through analogous notions of regularity and function class complexity.

\section{Survey of spectral GNNs for smoothing, bottlenecks, and communities}
\label{Ap:Survey of spectral GNNs for smoothing, bottlenecks, and communities}

\paragraph{Oversmoothing as spectral contraction} 
Oversmoothing limits the depth of GNNs and indicates that node embeddings contract toward a low-dimensional, low-frequency subspace as depth increases (See \cref{app:sec:oversmoothing}). It can be observed as the collapse of pairwise node-feature distances or decaying Dirichlet energy \citep{li2019deepgcns, cai2020note, rusch2023survey}.
Spectral filters often enforce non-amplifying assumptions that $\|h(\textbf{S})\|_2=\max_\lambda |h(\lambda)|=\gamma \leq 1$, which implies that after $t$ layers, the spectral components decay geometrically as $\| \textbf{X}^{(t)} \|\leq \gamma^t \|\textbf{X}\|$. Therefore, repeated low‑pass filtering collapses features into the span of the slowest-decaying eigenvectors.

Compared to spatial formulations, the frequency response $h(\lambda)$ makes the contraction rate and the role of the spectrum of the graph shift operator $\mathbf{S}$ explicit and exposes depth vs. contraction trade-offs directly through the shape of $h(\lambda)$ and the spectrum of $\mathbf{S}$. 
This also allows for understanding effective remedies for oversmoothing, as discussed, e.g., in \citet{rusch2023survey,epping2024graph,scholkemperresidual}. While spatial MPNNs can also be viewed as data-dependent low-pass operators, their effective spectrum is only implicitly defined by their local nonlinear aggregation and update functions, making these trade-offs harder to observe.


\paragraph{Oversquashing through bottlenecks} 
Oversquashing occurs when graph signals must traverse long distances through small cuts, creating structural bottlenecks for message passing \citep{alon2021bottleneck}. 
Deeper GNNs are then required to communicate across the cut, which exacerbates smoothing and attenuation. 
This effect has been formalized via feature Jacobians as a measure of pairwise influence, especially between distant nodes \citep{topping2021understanding, di2023over, black2023understanding, hariri2025return}, i.e., smaller upper bounds on the Jacobian norm indicate weaker communication. 
These bounds can be expressed through several measures such as balanced Forman curvature \citep{topping2021understanding}, effective resistance \citep{black2023understanding}, or commute time \citep{di2023over}. Each of these measures captures how well (or bad) two nodes are connected by short, parallel paths.

In the spectral view, bottlenecks manifest, e.g., as low algebraic connectivity (the Laplacian’s second eigenvalue $\lambda_2$) or poor conductance as captured by Cheeger-type bounds~\citep{jamadandi2024spectral}. 
These provide computable proxies for oversquashing.
Practically, one can diagnose oversquashing via $\lambda_2$, pairwise Jacobian norms, or the  histogram of eigenvalues together with the frequency response $h(\lambda)$. 
Mitigation strategies for oversquashing follow naturally: (i) we can design filters whose passbands preserve mid frequencies to support long-range communication without excessive depth (e.g., resolvents $h(\lambda)=(1-\alpha)/(1-\alpha\lambda)$, rational filters such as ARMA/Cayley, or mild band-pass), or (ii) rewire the graph to modify the spectrum, e.g., by increasing $\lambda_2$ or connecting high–effective-resistance pairs. 
Thus, spectral analysis reduces oversquashing diagnosis and repair to shaping $h(\lambda)$ and adjusting spectral properties of $\mathbf{S}$.


\paragraph{Community sensitivity from low‑frequency structure} On undirected graphs, community signals are those with small Dirichlet energy \citep{von2007tutorial}. Under variance and orthogonality constraints, the minimizers are the Laplacian’s low‑frequency eigenvectors, so the low‑eigenspace encodes clusters. A spectral layer with a low‑pass frequency response function preserves precisely this subspace while damping high‑frequency noise, yielding sample‑efficient, stable behavior on assortative graphs. Crucially, the spectral lens gives explicit, tunable control of the passband—via a single scale parameter in heat kernels $h(\lambda)=e^{-\tau\lambda}$, a resolvent $h(\lambda)=(1+\beta\lambda)^{-1}$, or a short Chebyshev expansion to align the cutoff and slope with the observed label spectrum. MPNN models can approximate the same effect only indirectly through depth, normalization, and residual connections that induce an implicit low‑pass whose cutoff emerges from training and graph scaling, making it harder to tune and certify. The spectral formulation is basis‑invariant, provides clear diagnostics, and extends naturally to multi‑scale communities via mixtures of diffusions, offering a cleaner, more controllable route to community sensitivity.

\section{Definition of oversmoothing}
\label{app:sec:oversmoothing}

For node features $\textbf{X}\in\mathbb{R}^{n\times d}$, the Dirichlet energy is defined as  
\begin{equation}
\mathcal{E}(\textbf{X})\;=\;\mathrm{tr}(\textbf{X}^\top \textbf{L} \textbf{X})\;=\;\frac12\sum_{(i,j)\in E} [\textbf{L}]_{ij}\,\|\textbf{X}_{:,i}-\textbf{X}_{:,j}\|_2^2,
\end{equation}
with a smaller $\mathcal{E}$ indicating smoothness across nodes. In the spatial form, the oversmoothing definition can be stated as 
\begin{definition}[spatial, Dirichlet‑energy oversmoothing]
The layer is Dirichlet‑contractive if there exists $\rho_t\in[0,1)$ with
$\mathcal{E}\big(H_t(\textbf{S})\textbf{X}\big)=\mathrm{tr}\big(\textbf{X}^\top H_t^\top \mathbf{L} H_t \textbf{X}\big)
\leq \rho_t^{2} \mathcal{E}(\mathbf{X})$ for all $\textbf{X}$. A sequence of layers oversmooths on $G$ if there is $\rho<1$ with $\sup_t \rho_t\le\rho$ and hence
\begin{equation}
\mathcal{E}\!\left(\mathbf{X}^{(t)}\right)\;\le\;\rho^{\,2t}\,\mathcal{E}\!\left(\mathbf{X}^{(0)}\right)\;\xrightarrow[t\rightarrow\infty]{}\;0.
\end{equation}

\end{definition}
Equivalently, embeddings converge (on each connected component) to constant vectors, making nodes indistinguishable.
In the spectral form of a graph filter, for a symmetric graph shift operator, i.e. adjacency matrix or Laplacian matrix, with eigenvalues $0\leq \lambda_1\leq \lambda_2\leq \cdots \leq \lambda_n$ written in an increasing order. Dirichlet energy can be further written as 
\begin{align}
\mathcal{E}\left(\textbf{X}^{(t+1)}\right)
&=\sum_{i=1}^n \lambda_i|h_t(\lambda_i)|^2\left\|\mathbf{v}_i^\top \mathbf{X}^{(t)}\right\|_2^2 \le\Big(\max_{i\ge2}|h_t(\lambda_i)|^2\Big)\,\mathcal{E}\left(\mathbf{X}^{(t)}\right).
\end{align}

\section{Computing spectral positional encodings with GNNs}\label{app:SPE}
\citet{kanatsoulislearning} show how message-passing and spectral GNN are cast as nonlinear functions of GSO eigenvectors. 
\begin{proposition}\citep{kanatsoulislearning}[GNNs are nonlinear functions of eigenvectors ]\label{prop:eig}
Consider a  GNN defined in Eq. (\ref{eq:mpnn_deepsets}) with $\kappa_t$ being the identity function and $\phi_t$ being a single-layer perceptron (SLP). The GNN recursive formula takes the form $\vec{X}^{(l)} =  \rho\left(\sum_{k=0}^{K-1}\bm S^k \vec{X}^{(l-1)}\vec{ H}_k^{(l)}\right)$, where $K=2$, and  {operates as} a nonlinear function of the GSO eigenvectors, i.e., {$\bm x_v^{(l)} =\texttt{SLP}\left(\bm v^{(v)}\right),~\bm v^{(v)}=\bm V[v,:]^T$}. The trainable parameters of the \texttt{SLP} are not independent but depend on the eigenvalues $\left\{\lambda_n\right\}_{n=1}^N$ and eigenvectors $\left\{\bm v_n\right\}_{n=1}^N$ of the GSO, as well as the node features $\bm X$ of the graph:
{\begin{equation}
\bm x_v^{(l)} =\rho\left(\bm W^T\bm v^{(v)}\right)
\end{equation}}
\begin{equation}\label{eq:dof}
\bm W\left[n, f\right] = \sum_{i=1}^{F_{l-1}}\sum_{k=0}^{K-1}\lambda_n^k\bm H_k^{(l)}[i,f]\langle {\bm \alpha_n}, \bm X^{(l-1)}\left[:,i\right]\rangle,
\end{equation}
{where $\bm\alpha_n = \bm v_n$ when the GSO is symmetric and $\bm\alpha_n = \bm V^{-1}[n,:]$ when it is not.} 
\end{proposition}

By Proposition~\ref{prop:eig}, the update at node~$v$ can be interpreted as applying a nonlinear transformation to its spectral representation~$\bm V[v, :]$. At the same time, the parameters of the single-layer perceptron define graph-dependent linear operators, namely
$\mathbf{W}[n,f] : (\mathcal{G}, \mathbb{R}^{F_{l-1}}) \rightarrow \mathbb{R}$.
The inner product $\langle \bm \alpha_n, \bm X^{(l-1)}[:, i] \rangle$ is determined by the eigenvectors and, when computing the update for node~$v$, aggregates information only from the features $\bm X^{(l-1)}[u, i]$ of neighboring nodes $u \in \mathcal{N}_v$.

\end{document}

%% file: macros.tex
\newcommand{\cX}{\mathcal{X}}

\newcommand{\Nb}{\mathbb{N}}

\newcommand{\Rb}{\mathbb{R}}

\newcommand{\new}[1]{\emph{#1}}

\renewcommand{\vec}[1]{\mathbold{#1}}

\DeclarePairedDelimiterX{\norm}[1]{\lVert}{\rVert}{#1}

\DeclareFontFamily{U}{mathx}{\hyphenchar\font45}
\DeclareFontShape{U}{mathx}{m}{n}{<-> mathx10}{}
\DeclareSymbolFont{mathx}{U}{mathx}{m}{n}
\DeclareMathAccent{\widebar}{0}{mathx}{"73}